%% file: neurips_2025.tex
\pdfoutput=1
\documentclass{article}

     \PassOptionsToPackage{numbers, compress}{natbib}

\usepackage[preprint]{neurips_2025}

\newcommand{\myparagraph}[1]{\noindent{\textbf{#1}~~}}

\usepackage[utf8]{inputenc} 
\usepackage[T1]{fontenc}    
\usepackage{hyperref}       
\usepackage{url}            
\usepackage{booktabs}       
\usepackage{amsfonts}       
\usepackage{nicefrac}       
\usepackage{microtype}      
\usepackage{xcolor}         
\usepackage{graphicx}
\usepackage{amsmath} 
\usepackage{booktabs}
\usepackage{bm}
\usepackage{fancybox}
\usepackage{subcaption}
\usepackage{wrapfig}
\usepackage{amsthm}

\newtheorem{lemma}{Lemma}[section]

\title{Revisiting LRP: Positional Attribution as the Missing
Ingredient for Transformer Explainability}

\author{%
  Yarden Bakish \\
  \small Tel-Aviv University\\
  \And
  Itamar Zimerman \\
  \small Tel-Aviv University \\
  \And
  Hila Chefer \\
  \small Tel-Aviv University \\
  \And
  Lior Wolf \\
  \small Tel-Aviv University \\
}

\begin{document}

\maketitle

\begin{abstract}
    The development of effective explainability tools for Transformers is a crucial pursuit in deep learning research. One of the most promising approaches in this domain is Layer-wise Relevance Propagation (LRP), which propagates relevance scores backward through the network to the input space by redistributing activation values based on predefined rules. However, existing LRP-based methods for Transformer explainability entirely overlook a critical component of the Transformer architecture: its positional encoding (PE), resulting in violation of the conservation property, and the loss of an important and unique type of relevance, which is also associated with structural and positional features. To address this limitation, we reformulate the input space for Transformer explainability as a set of position-token pairs
    . This allows us to propose specialized theoretically-grounded LRP rules designed to propagate attributions across various positional encoding methods, including Rotary, Learnable, and Absolute PE. Extensive experiments with {both fine-tuned classifiers and zero-shot foundation models, such as LLaMA 3,} demonstrate that our method significantly outperforms the state-of-the-art in both vision and NLP explainability tasks. Our code is publicly available.

\vspace{0.5em}
\hspace{.5em}
\includegraphics[width=1.25em,height=1.15em]{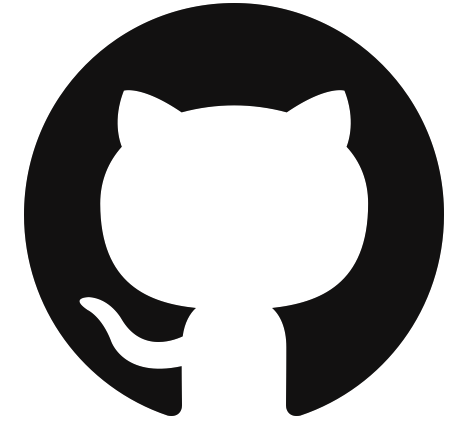}\hspace{.75em}
\parbox{\dimexpr\linewidth-7\fboxsep-7\fboxrule}{\url{https://github.com/YardenBakish/PE-AWARE-LRP}}
\vspace{-.5em}

\end{abstract}

\vspace{-7pt}
\section{Introduction\label{sec:intro}}
\vspace{-5pt}
Explainable AI (XAI) is increasingly vital in deep learning (DL), where models often achieve remarkable performance but operate as opaque ``black boxes''~\citep{arrieta2020explainable,das2020opportunities}. This lack of transparency reduces trust, limits user engagement, and complicates troubleshooting, thereby restricting the use of DL models in applications where decision-making transparency is essential. Consequently, developing XAI techniques for DL models has become an important research domain~\citep{samek2017explainable}. This task, however, is challenging, due to the inherent complexity of these models, which cannot be easily represented by simple functions.

Transformer-based architectures, which have become dominant in DL, present additional challenges for explainability due to their large scale, often containing billions of parameters. To address this, researchers have developed various attribution methods specifically designed for Transformers~\citep{chefer2021transformer,achtibat2024attnlrp,ali2022xai,abnar2020quantifying}. Among these, model-specific XAI techniques have gained prominence, providing explanations based on the model’s parameters, internal representations, and overall architecture.

The most effective model-specific XAI techniques, and the current state-of-the-art for Transformer explainability, are LRP-based, such as~\citep{achtibat2024attnlrp}. LRP is a well-established attribution technique that explains a model’s predictions by propagating relevance scores backward through the network, redistributing activation values based on predefined propagation rules. Unlike gradient-based methods, which often suffer from issues like vanishing gradients or numerical instabilities, LRP provides a more stable and precise way to trace how information flows through each layer.

\begin{wrapfigure}{r}{0.51\textwidth}
\centering
    \centering
    \begin{tabular}{@{\hskip 0.03in}c@{\hskip 0.03in}c@{\hskip 0.03in}c@{\hskip 0.03in}c@{\hskip 0.03in}} 
            \includegraphics[width=0.12\textwidth]{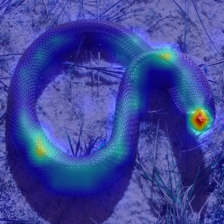} & 
        \includegraphics[width=0.12\textwidth]{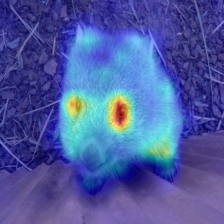} &
        \includegraphics[width=0.12\textwidth]{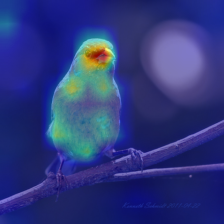} &
        \includegraphics[width=0.12\textwidth]{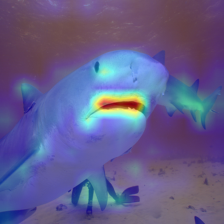} 
        \\
        \multicolumn{4}{c}{(a)} \\
                
        \includegraphics[width=0.12\textwidth]{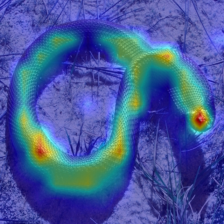} & 
        \includegraphics[width=0.12\textwidth]{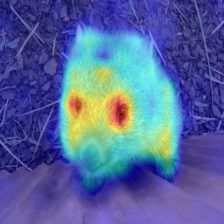} &
        \includegraphics[width=0.12\textwidth]{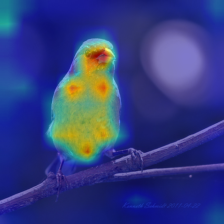} &
        \includegraphics[width=0.12\textwidth]{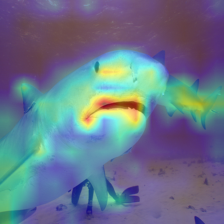} 
        \\
        \multicolumn{4}{c}{(b)}         
         \end{tabular}
\caption{\textbf{(a)} Explainability heatmaps by the state-of-the-art Attention-LRP method (AttnLRP)~\citep{achtibat2024attnlrp}. \textbf{(b)} The LRP heatmap obtained directly from our PE-aware LRP rules. 
The relevancy captured by the PE is less fragmented and captures more of the object.}
\vspace{-13pt}
\label{fig:Motivation}
\end{wrapfigure}
Recently, several refinements have been proposed to improve the stability and faithfulness of LRP rules for Transformers, leading to more robust and reliable interpretability techniques. Notable examples include~\citep{chefer2021transformer,ali2022xai} and \citep{achtibat2024attnlrp}, which introduce custom rules for propagating LRP through attention mechanisms, layer normalization, and other key components. Despite these advancements, we identify a critical gap in this extensive line of work: all existing LRP-based methods for Transformers overlook the need for PE-aware LRP rules and do not propagate attribution through positional encoding. This omission results in the loss of a key aspect of relevancy related to positional and structural concepts, limiting the ability to provide faithful and comprehensive explanations.


To mitigate this problem, we propose Positional-Aware LRP (PA-LRP), a novel technique that significantly improves upon previous methods through two fundamental modifications: (i) Reformulating the input space of the Transformer explainability problem to incorporate positional information. Instead of relying solely on the vocabulary space in NLP and the patch space in vision, we define the input space as a set of position-token pairs. (ii) Introducing the first LRP rules specifically designed to propagate relevance across standard positional encoding (PE) layers, including learnable PE, Rotary PE~\citep{su2024roformer}, and others. To enhance stability and faithfulness, our rules are further improved through techniques such as reparameterization of PE layers, linearization, and defining an appropriate sink for positional relevance to ensure that position-associated information is properly absorbed, which we validate to be crucial for precise propagation. 
Moreover, we provide complementary theoretical analysis, proving that our rules both satisfy the conservation property, and able to derive explanations more faithfully.

    \textbf{Our main contributions} 
    consist of the following: (i) We identify a critical gap in current LRP-based XAI techniques for Transformers: they overlook the attribution of positional encodings (PE). This omission results in a violation of the conservation property for input-level PE, as shown in Lemma~\ref{lem:lemma1},  and leads to unfaithful heatmaps when handling positional features, as demonstrated in Lemma~\ref{lem:lemma3}. We empirically validate that this omission is a critical limitation by significantly outperforming existing methods, and demonstrating that in certain cases, assigning relevance to PE alone can surpass standard state-of-the-art Transformer explainability techniques, showing that this signal is significant
    .
    Additionally, the obtained signal is complementary and distinct from the non-positional signal, better capturing spatial, positional, and structural relationships, as shown in Figure~\ref{fig:Motivation}. (ii) We introduce PA-LRP, a theoretically grounded and PE-aware technique for assigning relevance in Transformers. As shown in Tables~\ref{tab:pertubatoinNLp}--~\ref{tab:perturbationZeroShot}, PA-LRP significantly outperforms previous methods across both fine-tuned classifiers and zero-shot foundation models, in both NLP and vision tasks. (iii) Providing an open-source and user-friendly implementation of our method, along with demos and practical examples, to facilitate adoption by the broader research and practitioner community.
\vspace{-4pt}
\section{Background and Related Work\label{sec:relatedWork}}
\vspace{-4pt}
In this section, we describe the scientific context for discussing LRP-based Transformer explainability, along with the necessary terminology and symbols needed to describe our method.%
\vspace{-3pt}
\subsection{Positional Encoding in Transformers}
\vspace{-3pt}
Transformer-based ~\citep{vaswani2017attention} architectures rely on self-attention, which computes contextual relationships between tokens using:
{\small
\begin{equation}
    \text{Attention}(X) = \text{Softmax} \left( \frac{QK^T}{\sqrt{d_k}} \right) V,\quad Q = XW_Q,\quad K = XW_K,\quad V = XW_V
\end{equation}
}
%
where, $K, Q, V$ represent key, query, and value matrices respectively, $d_k$ is the embedding dimension, and $W_Q, W_K, W_V$ are learnable linear projection matrices. This attention mechanism is duplicated over several ``heads'' and is wrapped by standard DL peripherals such as Layer Normalization, FFNs, and skip connections, forming the core structure of a Transformer model by:
{\small
\begin{equation}
    X' = \text{LayerNorm} \left( X + \text{Attention}(X) \right), \quad
    X'' = \text{LayerNorm} \left( X' + \text{FFN}(X') \right) \,.
\end{equation}
}
where \textit{FFN} applies a two-layer linear transformation with activations in the middle.

Transformers operate on sets of tokens rather than ordered sequences, making them permutation-invariant by design. Unlike architectures with built-in order sensitivity such as RNNs~\citep{hochreiter1997long,jordan1997serial}, Transformers require explicit positional encoding (PE) to capture sequence structure. PE can be introduced at different stages of the model: it can be added to token embeddings at the input layer, as seen in learnable PE and sinusoidal PE~\citep{vaswani2017attention}, or integrated within the attention mechanism at each layer, as employed in Rotary PE (RoPE)~\citep{su2024roformer} and Alibi~\citep{presstrain}. The key insight of this paper is that while PE is well known for its important role in the forward pass~\citep{dufter2022position}, its crucial role in propagation-based XAI methods, such as LRP, has been largely overlooked, leading to violations of conservation and the loss of significant relevance, which often carries distinctive positional and structural meanings.

\noindent\noindent\textbf{Learnable PE.\quad} Learnable PE represents positions as trainable parameters, allowing the model to learn position representations directly from data. This approach offers flexibility and adaptability.

\noindent\noindent\textbf{Sinusoidal PE.\quad} Sinusoidal PE, as applied by the original Transformer model~\citep{vaswani2017attention}, encodes positions using sine and cosine functions with different non-trainable frequencies. Because it is based on absolute positions, it is less effective in tasks where relative positional information is more important.

\noindent\textbf{Rotary PE (RoPE).\quad} RoPE~\citep{su2024roformer} incorporates positional information by rotating token embeddings in a structured manner, enabling the model to naturally encode relative positions. Specifically, each key and query vector is transformed using a per-position block-diagonal rotation matrix. Unlike learnable or sinusoidal PEs, RoPE encodes relative positional relationships through the multiplication of rotation matrices. Due to its effectiveness, many popular foundation models, including SAM2~\citep{ravi2024sam},Pythia~\citep{biderman2023pythia}, LLaMA~\citep{touvron2023llama},  Qwen~\citep{bai2023qwen}, Gemma~\citep{team2024gemma}, and others are built on top or RoPE.

Other PE techniques, such as ALiBi~\citep{presstrain} and relative PEs~\citep{shaw2018self,raffel2020exploring}, are described in Appendix~\ref{app:addPE}.

\vspace{-3pt}
\subsection{Model-Specific XAI and LRP}
\vspace{-3pt}
Methods for explaining neural models have been extensively studied in the context of DNNs~\citep{zhang2021survey,linardatos2020explainable}, particularly in NLP~\citep{arras-etal-2017-explaining,yuan2021explaining} and computer vision~\citep{selvaraju2017grad,lang2021explaining}, and across various architectures including transformers~\citep{chefer2021generic,kokalj2021bert}, RNNs~\citep{arras2019explaining,zhang2021survey}, CNNs~\citep{ibrahim2023explainable,zhang2018visual}, state-space models~\citep{jafari2024mambalrp,ali2024hidden}, and others. A widely adopted strategy for this task is the use of model-specific techniques, which exploit the internal architecture and parameters of neural models to generate explanations. One notable method in this category is LRP~\citep{bach2015pixel,montavon2019layer}, which propagates relevance scores, denoted by $\mathcal{R}(\cdot)$, backwards through the network by redistributing activation values. Propagation relies on predefined rules and interactions between tokens.

\noindent\textbf{LRP.\quad}%
LRP is an evolution of gradient-based methods, such as Input $\times$ Gradient~\citep{shrikumar2017learning,baehrens2010explain}, which often suffer from issues like numerical instabilities and gradient shattering~\citep{balduzzi2017shattered}. LRP enhances backpropagation rules by enforcing two key principles: (i) the conservation property, which ensures that the total relevance is preserved across layers. Namely, for a layer $M$, where $Y=M(X)$, the relevance of the output $\mathcal{R}(Y)$ is equal to the relevance of the input $\mathcal{R}(X)$. (ii) The prevention of numerical instabilities during propagation. To achieve these goals, LRP rules are often derived from the Deep Taylor Decomposition principle~\citep{montavon2017explaining}, redistributing relevance scores at each layer based on the first-order Taylor expansion of the layer’s function.

\vspace{-5pt}
\subsection{XAI for Transformers} 
\vspace{-4pt}
The first model-specific XAI methods for Transformers were based on attention maps~\citep{caron2021emerging,clark2019does}, leveraging attention scores to quantify the contribution of each token to others across layers. Building on this approach,~\citet{abnar2020quantifying} introduced the attention rollout technique, which aggregates attention matrices across multiple layers to provide a more holistic explanation. However,~\citet{jain2019attention} later demonstrated that attention-based techniques can be misleading, as attention scores do not always correlate with gradient-based feature importance measures or actual model behavior. To address these limitations,~\citet{chefer2021transformer} developed a hybrid XAI method that combines LRP scores with attention maps, marking a breakthrough in the field by improving attribution fidelity. 

Alternatively, an extensive body of work focuses on model-agnostic methods~\citep{mosca2022shap,kokalj2021bert,cheng2025unifying}, however, these approaches often exhibit lower performance compared to model-specific techniques.

Purely LRP-based XAI methods for Transformers were first introduced in~\citep{voita2021analyzing} and later refined by~\citet{ali2022xai}, who developed custom LRP rules tailored for LayerNorm and attention layers to preserve conservation properties and ensure numerical stability. More recently,~\citet{achtibat2024attnlrp} further improved this approach by designing more faithful propagation rules for self-attention, achieving state-of-the-art performance in Transformer explainability. To the best of our knowledge, this represents the most advanced technique in the field and serves as our primary baseline. 

Interestingly, despite extensive research in this area, none of these approaches propagate relevance through PE layers. This omission leads to a loss of significant relevance associated with positional and structural features, ultimately resulting in less faithful and holistic attributions.

\vspace{-5pt}
\section{Method\label{sec:method}}
\vspace{-5pt}
\begin{wrapfigure}{r}{0.51\textwidth}
    \centering
    \vspace{-38pt}
    \includegraphics[width=0.91\linewidth]{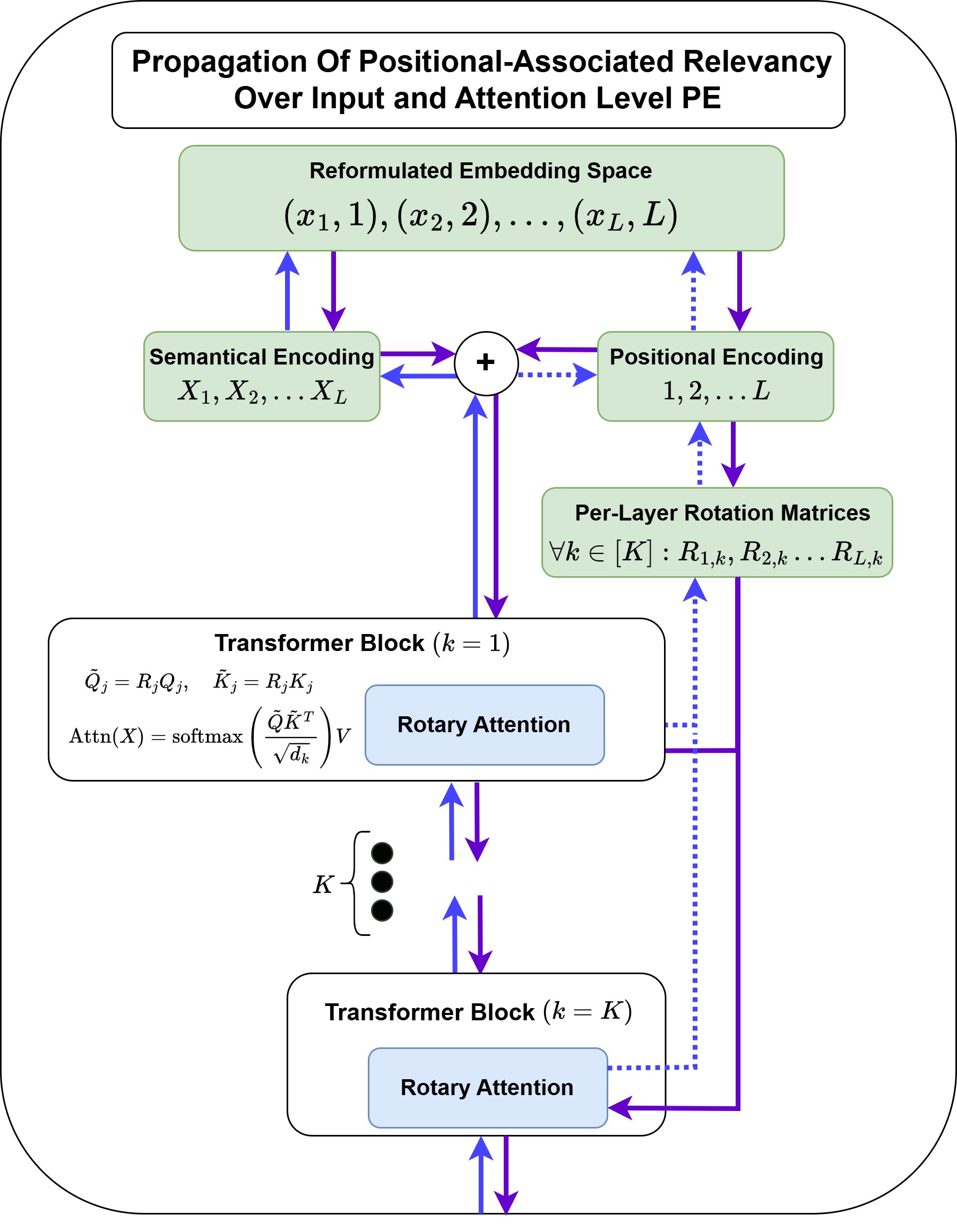}
    \vspace{-7pt}
    \caption{\small Visualization of our method for propagating PE-associated relevance. Purple arrows indicate the forward path, while blue arrows represent the LRP propagation rules. Dashed arrows denote custom PE-aware rules defined in our method.}
    \label{fig:methodFig}
    \vspace{-20pt}
\end{wrapfigure}

In this section, we describe our PE-aware LRP rules. We first revise the input space used in the Transformer explainability problem in Section~\ref{subsec:reform}. Then, building upon this formulation, we define our custom LRP rules in Section~\ref{subsec:PELRP-INPUT} and Section~\ref{subsec:PELRP-ATTN}. Finally, in Section~\ref{subsec:theory}, we prove that our PE-aware LRP rules are theoretically grounded.

\vspace{-3pt}
\subsection{Reformulating the Input Space\label{subsec:reform}}
\vspace{-3pt}

To comprehensively attribute positional information, we must define a \textbf{sink} that absorbs PE-associated relevance. To this end, we reformulate the explainability problem for Transformers. Given a sequence of embedded tokens of length $L$, denoted by $E_1, \ldots, E_L$
, where $D$ is the embedding size
, previous methods traditionally define the input space $\mathcal{S}$ as %
\begin{equation}
    \mathcal{S} = \{E_i \mid i \in [L], E_i \in \mathbb{R}^D\}\,.
\end{equation}
%
%
In contrast, we reformulate the input space as token-position pairs, with positional features defined separately for each layer, as follows:
\begin{equation}\label{eq:newInputSpace}
    \mathcal{S} = \Big{\{}\Big{(}E_i, (P_{i,1},P_{i,2},\ldots,P_{i,K})\Big{)} \mid i \in [L],\quad E_i \in \mathbb{R}^D,\quad P_{i,1},\ldots,P_{i,K} \in \mathbb{R}^{D'} \Big{\}}\,,
\end{equation}
where 
$D'$ is the dimension of the positional embeddings, and $K$ is the number of layers. Thus, in our formulation, each token in the input space consists of two ingredients: one representing the semantic embedding $E_i$ for all token indices $i\in [L]$, and the other representing the per-layer positional embedding $P_{i,k}$ for all layers $k\in [K]$.

The use of separate per-layer sinks for positional relevance ensures that the omission of certain positional features in one layer does not obscure or override essential features captured in other layers, and that important positional attributions are not discarded.
%
%





Building upon the formulation of Eq.~\ref{eq:newInputSpace}, the next two sections define LRP rules that enable stable propagation of relevance across PE components. In Section~\ref{subsec:PELRP-INPUT}, we present our rules for input-level PE, while Section~\ref{subsec:PELRP-ATTN} outlines our rules for attention-level PE.

\vspace{-4pt}
\subsection{LRP-rules for Input Level PE\label{subsec:PELRP-INPUT}}
\vspace{-3pt}

We begin with the simplest form of positional encoding—learnable PE—and then demonstrate that other input-level PEs can be reparameterized in a similar manner. Since PE is incorporated only at the input layer, we assume for brevity that $P_i$ is a vector rather than a matrix, namely $P_i = P_{i,1}$. We also tie the embedding dimensions of both the semantical and positional vectors ($D=D'$). 

\myparagraph{Learnable PE.} %
This layer learns positional information during training through a positional embedding matrix $P'\in \mathbb{R}^{L' \times D}$ where $D$ represents the embedding dimension of positional information, and $L'$ denotes the maximum sequence length. For each sample, the positional and semantic embeddings are summed to obtain the final input representation. Formally, the combined embedding for the token at position $i$ is given by $P'_i + E_i$ where $P'_i$ is the i-th row of the positional embedding matrix $P'$.
Thus, we can propagate relevance from the input of the first transformer block of the $i$-th token $z_i$ denoted by $\mathcal{R}(z_i)$, to the positional component $P'_i$ by using the standard LRP-$\epsilon$ rule for addition \citep{achtibat2024attnlrp}: %
%
{\vspace{-5pt}
\begin{equation}\label{eq:LernablePELRP}
    \text{PA-LRP for input-level PE}: 
    \mathcal{R}(P'_i) = P'_i\frac{\mathcal{R}(z_i)}{P'_{i}+E_{i} +\epsilon}\,. 
\end{equation}
}
%

\myparagraph{Sinusoidal PE.}%
This method encodes position information via a unique vector of sine and cosine values constructed by:
\begin{equation}\label{eq:sinPE}
    \text{Sinusoidal PE(i)}[2d] = \sin\left(\frac{i}{10000^{\frac{2d}{D}}}\right),\quad
\text{Sinusoidal PE(i)}[2d+1] = \cos\left(\frac{i}{10000^{\frac{2d}{D}}}\right)\,.
\end{equation}
Thus, the values derived from Eq.~\ref{eq:sinPE} can be used to reparameterize the positional embedding matrix $P'$, replacing the learned vectors with their corresponding sine and cosine values. Such reparameterization eliminates the need to propagate gradients through non-linear functions such as sine and cosine, improving efficiency and stability.

\vspace{-4pt}
\subsection{LRP-rules for Attention-level PE\label{subsec:PELRP-ATTN}}
\vspace{-3pt}
For attention-level PE, we focus on describing the PA-LRP rules for RoPE~\citep{su2024roformer} as a representative example. For the derivation of the PA-LRP rules for ALiBi~\citep{presstrain}, we refer the reader to Appendix~\ref{app:AlibiRules}.
At each layer $k$, RoPE modifies the queries ($Q$) and keys ($K$) matrices before computing the attention scores. This modification is done by multiplying each key and query vector by a position-dependent rotation matrix $R_{i,k} \in \mathbb{R}^{D \times D}$ where $i \in [L]$. The rotation matrix is a block-diagonal matrix defined as follows: %
{\small
\begin{equation}\label{eq:rotaryMatrices}
\forall i\in [L], k \in [K]: R_{i,k} =
\begin{bmatrix}
\cos \theta_i^{(1)} & -\sin \theta_i^{(1)} & \dots & 0 & 0 \\
\sin \theta_i^{(1)} & \cos \theta_i^{(1)}  & \dots & 0 & 0 \\
\vdots & \vdots & \ddots & \vdots & \vdots \\
0 & 0 & \dots & \cos \theta_i^{(D/2)} & -\sin \theta_i^{(D/2)} \\
0 & 0 & \dots & \sin \theta_i^{(D/2)} & \cos \theta_i^{(D/2)}
\end{bmatrix}
\end{equation}
}%
where each rotation angle \( \theta_i^{(m)} \) is defined as 
%
$\theta_i^{(m)} = i \omega_m$, where $\omega_i = 10000^{\frac{D}{2(m-1)}}$.

Note that in RoPE, as in other attention-level positional encodings, the positional information is represented by a matrix $R_{i,k}$, not a vector. Accordingly, we assume:
\begin{equation}\label{eq:flatten}
 P_{i,k} = \text{Flattening}(R_{i,k}), \quad D' = D^2\,.    
\end{equation}
Thus, we can propagate relevance from the matrix $\mathcal{R}(R_{i,k})$ to the vector $\mathcal{R}(P_{i,k})$ by flattening it: 
\begin{equation}\label{eq:flattenRelevance}
 \mathcal{R}(P_{i,k}) = \text{Flattening}(\mathcal{R}(R_{i,k}))\,.    
\end{equation}
Now, a key remaining step is to define how relevance should be propagated to $R_{i,k}$. The RoPE computation is executed before computing the attention scores, transforming the per-position queries and keys as follows: %
\begin{equation}\nonumber 
\forall i \in [L]
: \tilde{\bm{Q}}_i = R_{i,k} \bm{Q}_i, \quad \tilde{\bm{K}}_i = R_{i,k} \bm{K}_i,\quad
    \text{RoPE Attention}(X) = \text{Softmax} \left( \frac{\tilde{\bm{Q}}\tilde{\bm{K}}^T}{\sqrt{d_k}} \right) V\,.
\end{equation}
Our formulation builds on top of AttnLRP~\citep{achtibat2024attnlrp}, which propagates relevance over the queries $\tilde{\bm{Q}}$ and keys $\tilde{\bm{K}}$, resulting in their corresponding relevance scores $\mathcal{R}(\tilde{\bm{Q}})$,$\mathcal{R}(\tilde{\bm{K}})$. 
To propagate relevance from these matrices to the rotation matrices $R_{i,k}$, we apply the LRP rule for matrix multiplication employed by~\citet{achtibat2024attnlrp} separately to the key and query matrices, and then sum the resulting terms to produce a final attribution map per attention layer, as follows:
%
%
%
%
%
\begin{equation}\label{eq:PA-LRP-matmul}
   \forall i \in [L]
   : \mathcal{R}(R_{i,k}) = \frac{1}{2}\mathcal{R}(\tilde{\bm{Q_{i}}}) + \frac{1}{2}\mathcal{R}(\tilde{\bm{K_{i}}})\,.
\end{equation}

Finally, we aggregate the relevance scores across K layers and across feature dimension D, summing only the positive contributions, following a strategy similar to that proposed by~\citet{chefer2021transformer} and~\citet{xiong2024explainableartificialintelligencexai}:
\begin{equation}\label{eq:aggrPositive}
    \mathcal{R}_i = \sum_{d \in [D]} E_i [d]^+ + \sum_{k \in [K]}\sum_{d' \in [D']} P_{i,k}[d]^+ \,,
\end{equation}
where $\mathcal{R}_i$ is the final relevance for token $i$, and $(\cdot)^+$ denotes the ReLU function, which filters out negative values. %

\myparagraph{Overall Method.} %
Our PA-LRP rules allow us to assign relevance to the positional part of the input space. For the non-positional part, we use the same rules as defined in AttnLRP~\citep{achtibat2024attnlrp}.
It is worth noting that, although our rules in Eqs.\ref{eq:LernablePELRP},\ref{eq:flattenRelevance},\ref{eq:PA-LRP-matmul}, and \ref{eq:aggrPositive} are built on top of the AttnLRP framework, they are not limited to it. Our input-level PE rules can be decoupled and applied to any LRP method, while the attention-level PE rules can be integrated with alternative formulations, as long as they propagate relevance through the attention matrices and preserve the connection between PE and the computational graph.

As a result, similar to other LRP methods, our approach can produce explainability maps with computational efficiency comparable to a single backward pass. We further clarify that although our method introduces several modifications in the forward path and input space, it does not require any changes to the transformer itself. Instead, these modifications propose an equivalent forward path that allows us to better define the propagation rules.

\vspace{-4pt}
\subsection{Theoretical Analysis\label{subsec:theory}}
\vspace{-3pt}
To support our PA-LRP rules, we provide theoretical evidence demonstrating that they satisfy the key LRP criteria. First, the following two lemmas prove that our proposed LRP rules satisfy the conservation property. 

\begin{lemma}\label{lem:lemma1}
\textit{For input-level PE transformers, the conservation property is violated when disregarding the positional embeddings' relevancy scores.}
\end{lemma}
\begin{lemma}\label{lem:lemma2}
\textit{For attention‐level PE transformers, our PE‐LRP rules satisfy the conservation property.}
\end{lemma}
Next, we present a lemma based on a key example illustrating that existing methods exhibit low faithfulness. In particular, we show that in simplified settings, applying standard LRP techniques without incorporating position-aware LRP rules leads to unfaithful explanations in tasks that heavily depend on positional features. 
\begin{lemma}\label{lem:lemma3}
\textit{For attention‐level PE transformers, current LRP attribution rules achieve low faithfulness, especially when considering positional features.}
\end{lemma} %
The proofs and examples are detailed in Appendix~\ref{app:Proofs}.
\vspace{-4pt}
\section{Experiments}
\vspace{-4pt}
To assess the effectiveness of our PA-LRP rules, we conduct a comprehensive set 
of experiments in both the NLP and vision domains. First, in Section~\ref{sec:resultsNLP}, we perform perturbation tests with both zero-shot foundation models and finetuned classifiers, as well as performing an ablation study. Next, in Section~\ref{sec:resultsVision}, we conduct perturbation and segmentation tests in the Vision domain using DeiT~\citep{touvron2021training}.


We begin by describing our baselines, ablation variant, and evaluation metrics:

\noindent\textbf{Baseline and Ablation Variant.}\quad %
Our primary baseline for comparison is AttnLRP~\citep{achtibat2024attnlrp}, as it represents the SoTA in general transformer XAI, and our method builds on top of it for non-positional components. The key distinction between our approach and this baseline (as well as other LRP-based methods) is our ability to attribute relevance to positional information. Our composite approach that balances both positional and non-positional relevance is denoted as PA-LRP, or 'ours'. Additionally, to isolate the effect of the positional encoding, we introduce an ablation variant denoted by 'PE Only', which directly measures the relevance assigned to positional components at the input space using our custom attribution rules.

Although empirical evaluation of attribution methods is inherently challenging, we validate our PA-LRP method using perturbation and segmentation tests. Below, we describe these metrics:

\noindent\textbf{Perturbation Tests.}\quad %
 Perturbation tests are split into two metrics: positive and negative perturbations, which differ in the order in which pixels or tokens are masked. In positive perturbation, pixels or tokens are masked in descending order of relevance. An effective explanation method identifies the most influential regions, leading to a noticeable drop in the model’s score (measured in comparison to the predicted or target class) as these critical areas are gradually removed. In negative perturbation, masking begins with the least relevant elements and progresses toward the more important ones. A reliable explanation should keep the model’s prediction stable, demonstrating robustness even when unimportant components are masked.

 Following~\citep{ali2022xai,zimerman2025explaining}, in both Vision and image domains, the final metric is quantified using the Area-Under-Curve (AUC), capturing model accuracy relative to the percentage of masked pixels or tokens, from 10\% to 90\%.

\noindent\textbf{Segmentation Tests.}\quad%
For attribution methods in vision, segmentation tests are a set of evaluations used to assess the quality of a model’s ability to distinguish foreground from background in an image. %
These tests compare the labeled segmentation image, which indicates whether each pixel belongs to the background or the foreground, with the explainability map, after it has been binarized using a thresholding technique. Then, several metrics are computed over both images: (i) Pixel Accuracy: The percentage of correctly classified pixels, measuring how well the predicted segmentation aligns with the ground truth. (ii) Mean Intersection-over-Union (mIoU): The ratio of the intersection to the union of the predicted and ground-truth segmentation maps, averaged across all images. (iii) Mean Average Precision (mAP): A metric that considers precision and recall trade-offs at different thresholds, providing a robust assessment of segmentation quality.

\vspace{-4pt}
\subsection{Results in NLP\label{sec:resultsNLP}}
\vspace{-3pt}
For experiments in the NLP
, we first present results for perturbation tests, including an ablation study. For our tests, we adopt settings defined in~\citep{ali2022video,zimerman2025explaining} and we present qualitative results in Appendix.~\ref{sec:additonalNLPExamples}.

\myparagraph{Perturbation Tests for Finetuned Models.} %
We conduct perturbation tests on three LLMs, finetuned on the IMDB classification dataset: LLaMa 2-7B~\citep{touvron2023llama}, LLaMa 2-7B Quantized, and Tiny-LLaMa~\citep{zhang2024tinyllama}.
The results presented in Table~\ref{tab:pertubatoinNLp} demonstrate that our method achieves better scores than the LRP baseline across all metrics 
. In particular, our approach improves the AU-MSE score in the generation scenario by 14.5\% for LLaMa 2-7B, 10.6\% for LLaMa 2-7B Quantized, and 51.41\% for Tiny-LLaMa.

\begin{table}[t]\centering 
\small
\setlength{\tabcolsep}{3pt}
\caption{\small \textbf{Perturbation Tests in NLP.} Evaluation of LLaMa-2 7B, Quantized LLaMa-2 7B, and Tiny-LLaMa, all finetuned on IMDB, on pruning and generation perturbation tasks. AttnLRP~\citep{achtibat2024attnlrp} is the  LRP baseline. The metrics used are AUAC (area under activation curve, higher is better) and AU-MSE (area under the MSE, lower is better).\label{tab:pertubatoinNLp}}
\resizebox{0.66\textwidth}{!}{
\begin{tabular}{c c c c c c }
\toprule
\multicolumn{1}{c}{Model}  & \multicolumn{1}{c}{Method}  & \multicolumn{2}{c}{Generation} & \multicolumn{2}{c}{Pruning}  \\ 
& & \multicolumn{1}{c}{AUAC $\uparrow$} & \multicolumn{1}{c}{AU-MSE $\downarrow$} & \multicolumn{1}{c}{AUAC $\uparrow$} & \multicolumn{1}{c}{AU-MSE $\downarrow$}\\
 \midrule
%
LLaMa-2 7B & AttnLRP& 0.779& 7.629& 0.777& 6.548\\ 
LLaMa-2 7B & PE Only& 0.771& 6.792& 0.771& 6.823\\
LLaMa-2 7B & Ours& \textbf{0.796}& \textbf{6.521}& \textbf{0.790}& \textbf{6.325}\\
\midrule
LLaMa-2 7B  Quantized & AttnLRP& 0.774& 11.348&  0.767& 10.067\\ 
LLaMa-2 7B Quantized & PE Only& 0.758&  10.730& 0.758& 10.774\\
LLaMa-2 7B Quantized & Ours& \textbf{0.785}& \textbf{10.137}& \textbf{0.778}& \textbf{9.685}\\
%
\midrule
Tiny-LLaMa-2 7B & AttnLRP& 0.803& 8.065& 0.792& 4.030\\ 
Tiny-LLaMa-2 7B & PE Only& 0.788& \textbf{3.918}& 0.788& \textbf{3.947}\\
Tiny-LLaMa-2 7B & Ours& \textbf{0.806}& {4.915}& \textbf{0.805}& 4.082\\ 
\bottomrule\\
\end{tabular}
}
\vspace{-11pt}

\end{table}

\myparagraph{Perturbation Tests in Zero-Shot Settings.} 
We use LLaMa 3-8B ~\citep{grattafiori2024llama} to evaluate explainability performance in zero-shot setting.
The results presented in Table~\ref{tab:perturbationZeroShot}  showcase the superiority of our method across all metrics. \textbf{(i) Multiple-Choice Question Answering (MCQA):} our approach improves, on both generation and pruning scenarios, the AUAC score by approximately 3.2\%, and AU-MSE score by approximately by 7.7\%. 
\textbf{(ii) Next Token Prediction:} our approach improves the AUAC score by approximately 0.5\% on both generation and pruning scenarios, and AU-MSE score by approximately by 3\% on both scenarios. In contrast to MCQA, the Wikipedia dataset consists relatively long texts, making shifts in relevancy distributions less critical to the model's prediction.

\begin{wraptable}{r}{0.50\textwidth} 
\vspace{-9pt}
\small
\setlength{\tabcolsep}{3pt}
 \caption{\small \textbf{Ablation Study:} Analyzing the contribution of the multi-sink mechanism via perturbation tests in NLP. The evaluation was conducted on 
 the IMDB dataset.\label{tab:pertubatoinAblationNLp}}
 \vspace{-3pt}
 \resizebox{0.5\textwidth}{!}{
\begin{tabular}{c c c c c }
\toprule
\multicolumn{1}{c}{Method}  & \multicolumn{2}{c}{Generation} & \multicolumn{2}{c}{Pruning}  \\ 
& \multicolumn{1}{c}{AUAC $\uparrow$} & \multicolumn{1}{c}{AU-MSE $\downarrow$} & \multicolumn{1}{c}{AUAC $\uparrow$} & \multicolumn{1}{c}{AU-MSE $\downarrow$}\\
 \midrule
Ours& \textbf{0.796}& \textbf{6.521}& \textbf{0.790}& \textbf{6.325} \\
 w/o Multi-Sink &  0.759 & 7.124 & 0.758 & 7.158
 \\ 
\bottomrule\\
\end{tabular}
}
\vspace{-17pt}
\end{wraptable}
\myparagraph{Ablation.} 
To better understand the contribution of our PA-LRP rules, we conduct perturbation tests for the method that attributes solely position-associated relevance. The results are presented in the second, fourth, and sixth rows of Table~\ref{tab:pertubatoinNLp}, and second row of Table~\ref{tab:perturbationZeroShot}, and are denoted by 'PE-Only.' Surprisingly, this method produces results similar to the AttnLRP baseline, demonstrating the importance of PE-associated relevance, which carries a significant part of the signal. Notably, in Table ~\ref{tab:pertubatoinNLp}, this variant achieves the best score on the AU-MSE metric for Tiny-LLaMA, reducing the error by 50\% compared to AttnLRP~\citep{achtibat2024attnlrp}. Moreover, in Table~\ref{tab:pertubatoinAblationNLp}, we ablate the contribution of our multi-sink approach, which relies on drawing solely positive contributions across layers, as we aim to prevent the loss of positional relevance due to influence of negative contributions in final layers. We evaluate explainability performance for binary classification 
of LLaMa-2-7B, using the same perturbation metrics, and report that the multi-sink approach improves the results by 7\%.

\begin{table}[t]\centering 
\small
\setlength{\tabcolsep}{4pt}
\caption{\small \textbf{Perturbation Tests in NLP (Zero-Shot).} Evaluation of LLaMa-3 8B in zero-shot on generation and pruning perturbation tasks for both multiple-choice question answering and Next-Token Prediction (NTP) settings. Metrics reported are AUAC (area under activation curve, higher is better) and AU-MSE (area under MSE, lower is better). ``AttnLRP'' refers to the LRP baseline~\citep{achtibat2024attnlrp}. 'G' for generation and 'P' for pruning.} 
\resizebox{\textwidth}{!}{
\begin{tabular}{c 
c c c c 
c c c c}
\toprule
\multicolumn{1}{c}{} &
\multicolumn{4}{c}{Multiple-Choice Question Answering} &
\multicolumn{4}{c}{Next Token Prediction} \\
\cmidrule(lr){2-5} \cmidrule(lr){6-9}
Method 
& G. AUAC $\uparrow$ & G. AU-MSE $\downarrow$ & P. AUAC $\uparrow$ & P. AU-MSE $\downarrow$
& G. AUAC $\uparrow$ & G. AU-MSE $\downarrow$ & P. AUAC $\uparrow$ & P. AU-MSE $\downarrow$ \\
\midrule

AttnLRP 
& 0.365 & 66.399 & 0.354 & 68.856 
& 0.559 & 41.704 & 0.559 & 42.003 \\

PE Only 
& 0.374 & \textbf{61.014} & 0.364 & \textbf{63.141} 
&0.557 & 40.538 & 0.556 & 40.800 \\

Ours & \textbf{0.377} & 61.285 & \textbf{0.368} & 63.424 
& \textbf{0.562} & \textbf{40.474} & \textbf{0.561} & \textbf{40.735} \\
\bottomrule\\
\end{tabular}
}
\label{tab:perturbationZeroShot}
\vspace{-12pt}
\end{table}

\vspace{-3pt}
\subsection{Results for  Vision Transformers\label{sec:resultsVision}}

\begin{wrapfigure}{r}{0.48\textwidth}
\vspace{-45pt}
\centering
\begin{tabular}{@{\hskip 0.03in}c@{\hskip 0.03in}c@{\hskip 0.03in}c@{\hskip 0.03in}c@{\hskip 0.03in}} 
 \includegraphics[width=0.12\textwidth]{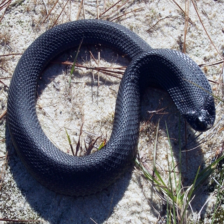} & 
    \includegraphics[width=0.12\textwidth]{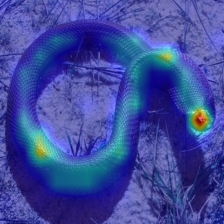} &
    \includegraphics[width=0.12\textwidth]{figures/vision/img1/196_basic_custom_lrp_PE_ONLY.png} & \includegraphics[width=0.12\textwidth]{figures/vision/img1/196_basic_custom_lrp_SEMANTIC_ONLY.png}
\\
  \includegraphics[width=0.12\textwidth]{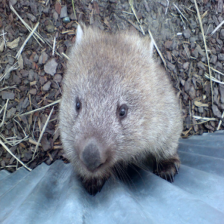} & \includegraphics[width=0.12\textwidth]{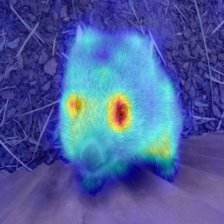} &
  \includegraphics[width=0.12\textwidth]{figures/vision/img5/165_basic_custom_lrp_PE_ONLY.png} &
\includegraphics[width=0.12\textwidth]{figures/vision/img5/165_basic_custom_lrp_SEMANTIC_ONLY.png}
 \\

    \includegraphics[width=0.12\textwidth]{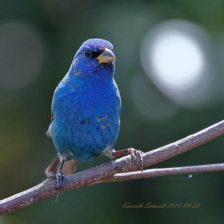} & 
    \includegraphics[width=0.12\textwidth]{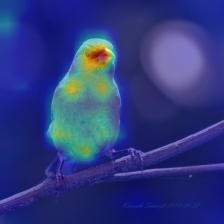} &
    \includegraphics[width=0.12\textwidth]{figures/vision/img17/img17_pos.png} & \includegraphics[width=0.12\textwidth]{figures/vision/img17/img17_sem.png} \\

  \includegraphics[width=0.12\textwidth]{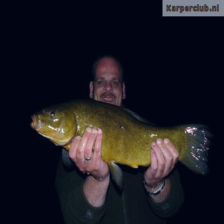} & \includegraphics[width=0.12\textwidth]{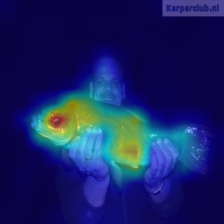} &
  \includegraphics[width=0.12\textwidth]{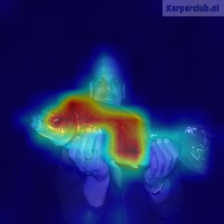} &
\includegraphics[width=0.12\textwidth]{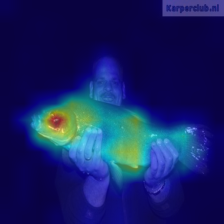}
 \\
 \includegraphics[width=0.12\textwidth]{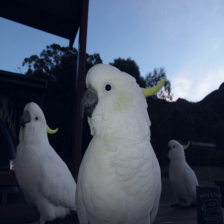} & \includegraphics[width=0.12\textwidth]{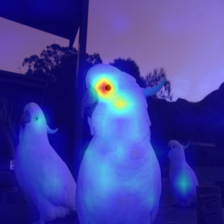} &
  \includegraphics[width=0.12\textwidth]{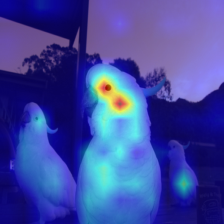} &
\includegraphics[width=0.12\textwidth]{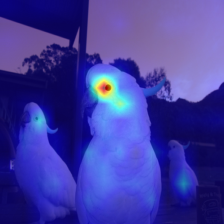}
 \\
 \includegraphics[width=0.12\textwidth]{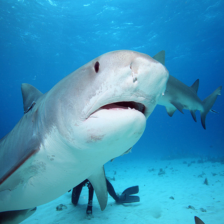} & \includegraphics[width=0.12\textwidth]{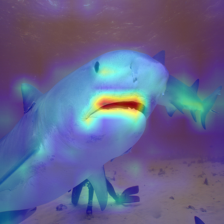} &
  \includegraphics[width=0.12\textwidth]{figures/vision/img6/157_basic_custom_lrp_PE_ONLY.png} &
\includegraphics[width=0.12\textwidth]
{figures/vision/img6/157_basic_custom_lrp_SEMANTIC_ONLY.png} \\
  \includegraphics[width=0.12\textwidth]{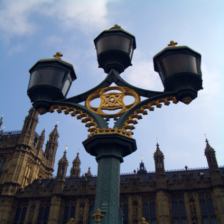} & \includegraphics[width=0.12\textwidth]{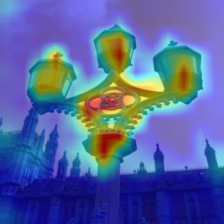} &
  \includegraphics[width=0.12\textwidth]{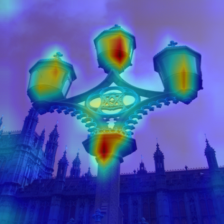} &
\includegraphics[width=0.12\textwidth]{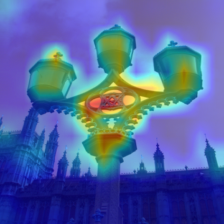}\\
\includegraphics[width=0.12\textwidth]{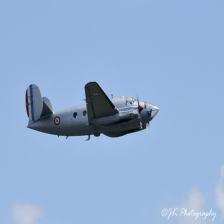} & 
\includegraphics[width=0.12\textwidth]{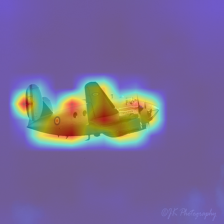} &
\includegraphics[width=0.12\textwidth]{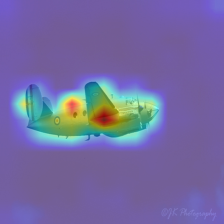} & \includegraphics[width=0.12\textwidth]{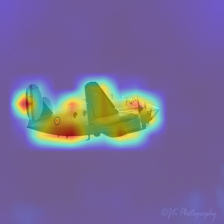} \\(a) & (b) & (c) & (d)
 \end{tabular}
 \vspace{-6pt}
 \caption{\small Results of different explanation methods for DeiT. (a) Input image. (b) PA-LRP (ours), which includes PE attribution. (c) PE only LRP, (d) AttnLRP~\citep{achtibat2024attnlrp}, which does not attribute relevancy to PE.}\label{fig:qualitative_cmp_vision}
 \vspace{-44pt}
\end{wrapfigure}

For vision models, we present both quantitative and qualitative analysis. 

\myparagraph{Qualitative Analysis.}
For qualitative analysis, we visualize the explainability maps obtained from our method, the AttnLRP~\citep{achtibat2024attnlrp} baseline, and the ablation variant that focuses exclusively on PE-associated relevance, denoted by PE Only. Additional examples with larger images are presented in appendix~\ref{sec:additonalExamples}.

Figure~\ref{fig:qualitative_cmp_vision} presents a comparative visualization of these maps. The results reveal three notable trends. \noindent\textbf{(i) Effectiveness of PE-associated relevance:} The maps from the ablation variant perform at the same level as the AttnLRP baseline. This finding highlights the strength of our method in identifying important signals that previous works have overlooked, underscoring the importance of our PA-LRP rules. \noindent\textbf{(ii) The uniqueness of PE-associated relevance:} The attributed signal derived solely from positional-associated relevance 
captures unique relationships, exhibiting clearer spatial and structural patterns. In particular, relevance is distributed across the entire object, especially in the snake, bird, and shark examples. In contrast, the baseline method, which does not propagate relevance through PEs, produces a sparser pattern that does not focus on the entire object but instead is highly selective to specific patches. One possible explanation is that positional-associated relevance better captures concepts related to position, structure, order, and broader regions within the image. \noindent\textbf{(iii) The importance of balancing:} It is evident that the maps obtained from the PE-associated method and the baseline are complementary, and their combination, extracted via our approach, provides the most robust explanations.

\myparagraph{Quantitative Analysis.}
Here, we present our quantitative results through perturbation and segmentation tests.

\noindent\textbf{Perturbation Tests in Vision.}\quad %
The results for perturbation tests are shown in Table~\ref{tab:visionPertubatoin}, where we compare our method against the attention LRP baseline and an ablation variant that focuses solely on maps obtained for positionally associated relevance (PE only). Experiments are conducted using three model sizes: Tiny, Small, and Base.

\begin{table}[h]
\centering
\setlength{\tabcolsep}{3pt}
\begin{minipage}[t]{0.435\textwidth}
\centering
\caption{\small Perturbation Tests for DeiT Variants on ImageNet. AUC results for predicted class. Higher (lower) is better for negative (positive).}
\small
\begin{tabular}{c c c c c c }
\toprule
\multicolumn{1}{c}{M. Size}& \multicolumn{1}{c}{Method}& \multicolumn{2}{c}{Negative $\uparrow$ } & \multicolumn{2}{c}{Positive $\downarrow$} \\ 
 & & Predicted & Target & Predicted & Target    \\ 
\midrule
Base & AttnLRP & 52.185 & 47.516 & 10.784 & \textbf{8.032} \\ 
Base & Ours & \textbf{54.970} & \textbf{50.174} & \textbf{9.918} & 9.237 \\
\midrule
Small & AttnLRP & 50.662 & 45.105 & 10.511 & 9.761 \\ 
Small & Ours & \textbf{53.482} & \textbf{47.948} & \textbf{9.135} & \textbf{8.477} \\
\midrule
Tiny & AttnLRP & 43.832 & 37.499 & \textbf{2.796} & \textbf{2.503} \\ 
Tiny & Ours & \textbf{50.126} & \textbf{42.567} & 3.579 & 3.214 \\
\bottomrule \\
\end{tabular}

\label{tab:visionPertubatoin}
\end{minipage}
\hfill
\begin{minipage}[t]{0.46\textwidth}
\centering
\small
\caption{\small Segmentation performance of DeiT variants on ImageNet segmentation~\citep{guillaumin2014imagenet}. 'M.' for model. Higher is better. }
\begin{tabular}{c c c c }
\toprule
M. Size & Method & Pixel Accuracy $\uparrow$ & mIoU $\uparrow$\\ 
\midrule
Base & AttnLRP & 72.204 & 50.100 \\ 
Base & Ours & \textbf{72.698} & \textbf{51.400} \\ 
\midrule
Small & AttnLRP & 72.114 & 50.000 \\ 
Small & Ours & \textbf{73.060} & \textbf{51.700} \\ 
\midrule
Tiny & AttnLRP & 74.815 & 52.850 \\ 
Tiny & Ours & \textbf{76.613} & \textbf{55.920} \\
\bottomrule \\
\end{tabular}

\label{tab:visionSegmentation}
\end{minipage}
\vspace{-5pt}
\end{table}

Notably, our method outperforms the baseline by a significant margin. For instance, in negative perturbation of the predicted class, our method improves the performance by an average of 3.97 points across the three model sizes. However, in positive perturbation, our method lags behind the baseline in half of the cases, though by a small margin of at most 1.2 points.

\vspace{-5pt}
\section{Discussion: The Role of Attributing PEs}
\vspace{-5pt}
Our {theoretical and} empirical analysis suggests that both semantic and positional relevance are complementary, and combining them is essential to provide precise explanations. LRP-type attribution creates pixel-level heatmaps, but can we characterize and identify which {\em concepts} are attributed mainly by positional relevance versus semantic relevance?

We may expect, for example, that objects that are usually placed in specific contexts (boats on water, airplanes in the sky) would display a more significant PE component. 
Much of this position context is relative. RoPE, for example, captures relative position through the matrix multiplication of two position-dependent rotation matrices, which plays a fundamental role in capturing spatial features in vision tasks (e.g., objects spanning across multiple patches) and when modeling relationships between words in the same sentence in NLP. In such cases, our PA-LRP rules can effectively attribute positional features, that are largely ignored by standard LRP methods. 


\vspace{-6pt}
\section{Conclusions}\label{sec:conclusion}
\vspace{-6pt}
This paper explores the importance of assigning LRP scores to positional information, a crucial component of Transformers and LLMs. Our theoretical and empirical analysis demonstrates that positional-associated relevance carries a unique type of significance and can drastically improve XAI methods for attention models. %

Regarding limitations, we emphasize that our work focuses on designing new custom LRP rules to propagate relevance through PEs, leveraging the insight that this aspect has been previously overlooked. However, we do not extend this insight to redesign or systematically revisit existing LRP rules. Such a redesign could offer an opportunity to empirically and theoretically establish improved LRP rules for attention mechanisms and Transformer models.

\section{Acknowledgments}
This work was supported by a grant from the Tel Aviv University Center for AI and Data Science (TAD). This research was also supported by the Ministry of Innovation, Science \& Technology ,Israel (1001576154) and the Michael J. Fox
Foundation (MJFF-022407).

\newpage

\bibliographystyle{plainnat}
\bibliography{references}
\newpage
\appendix

\section{Background for Additional PEs\label{app:addPE}}
In this appendix, we introduce additional PEs beyond those presented in Section~\ref{sec:relatedWork}.

\noindent\textbf{Relative Positional Bias (RPB).\quad}%
Similar to Alibi, RPB~\citep{raffel2020exploring} modifies the attention scores by introducing a learnable bias term that depends on the relative distance between query and key tokens. For a query at position $i$ and a key at position $j$, the attention scores are adjusted as follows:
\begin{equation}
    A_{i,j} = A_{i,j} + B(|i-j|)
\end{equation}
where $B(i-j)$ is a learned bias function that depends only on the relative position difference $(i-j)$, rather than the absolute positions.

\noindent\textbf{Attention with Linear Biases (ALiBi).\quad}%
ALiBi~\citep{presstrain} is a positional encoding method designed to help transformers generalize to longer sequences when trained on shorter ones. Instead of using explicit positional embeddings, ALiBi modifies attention scores directly by introducing a learned linear bias that penalizes attention weights based on token distance.

Specifically, for a query token at position $j$, Alibi adjusts the attention scores as follows:
\begin{equation}
    A'_{i,j} = A_{i,j} + m (|i-j|)
\end{equation}
where $m$ is a learned or predefined slope that controls how quickly attention strength decays with distance. Different attention heads can use different slopes, enabling some heads to focus more on local interactions while others capture long-range dependencies.

\section{PA-LRP Rules for Alibi\label{app:AlibiRules}}
 
Recall the main modification in the ALiBi computation:
 \begin{equation}
    A'_{i,j} = A_{i,j} + P_{i,j}, \text{ where } P_{i,j}= m (|i-j|)
\end{equation}
Adopting the same approach presented for RoPE, given the relevancy scores of $A'_{i,j}$, denoted by $\mathcal{R}(A'_{i,j})$, we define specialized rules to propagate relevancy from $A'_{i,j}$ to the positional terms of ALiBi at each layer, namely, indices $i$ and $j$. We begin by distributing the relevancy scores between $A_{i,j}$ and $P_{i,j}$, using the standard LRP-$\epsilon$ rule for addition, giving us:
{\vspace{-5pt}
\begin{equation} 
    \mathcal{R}(P_{i,j}) = P_{i,j}\frac{\mathcal{R}(A'_{i,j})}{A_{i,j}+P_{i,j} +\epsilon} 
\end{equation}
}
We proceed to propagate the relevancy scores $\mathcal{R}(P_{i,j})$ to the positional encoding $i$ and $j$ in a similar fashion to our rules for RoPE. We make the following observations:  (i) $m$ is a constant, resulting in 100\% of the relevancy to propagate from $P_{i,j}$ to $|i-j|$. (ii) Since we are using auto-regressive models, we get that $i>j$, allowing us to ignore the absolute value function (iii) The standard LRP-$\epsilon$ rule for addition applies the same of subtraction, as we can express $i-j$ as $i+(-j)$, and also $-j =(-1)\cdot j$, and since $-1$ is constant, we propagate the entire relevancy to $j$. That gives us:
{\vspace{-5pt}
\begin{equation} 
    \mathcal{R}(i) = i\frac{\mathcal{R}(P_{i,j})}{i+(-j)+ \epsilon}\text{,} \quad \mathcal{R}(j) = \mathcal{R}(-j) = j\frac{\mathcal{R}(P_{i,j})}{i+(-j)+ \epsilon}
\end{equation}
}
From hereon we adhere to our PA-LRP rules, aggregating the relevance scores of the positional terms across all layers as employed in Section~\ref{subsec:PELRP-ATTN}.

\section{Additional Qualitative Results in Vision\label{sec:additonalExamples}}
In addition to Figure~\ref{fig:qualitative_cmp_vision}, we provide more examples in Figures \ref{fig:attVisExamples1} -~\ref{fig:attVisExamples3}. As previously explained, PE-associated relevance better highlights the entire object, and overcomes the issue of over-consideration of the foreground, where extremely high relevancy scores are produced for patches which are more concerned with semantics or common patterns, like a bird's beak in the first row in Figure~\ref{fig:attVisExamples2}.

\begin{figure*}[t]
\centering
\begin{tabular}{@{\hskip 0.03in}c@{\hskip 0.03in}c@{\hskip 0.03in}c@{\hskip 0.03in}c@{\hskip 0.03in}} 
    \includegraphics[width=0.20\textwidth]{figures/vision/img1/img_196.png} & 
    \includegraphics[width=0.20\textwidth]{figures/vision/img1/196_basic_custom_lrp.png} &
    \includegraphics[width=0.20\textwidth]{figures/vision/img1/196_basic_custom_lrp_PE_ONLY.png} & \includegraphics[width=0.20\textwidth]{figures/vision/img1/196_basic_custom_lrp_SEMANTIC_ONLY.png}
 \\
  \includegraphics[width=0.20\textwidth]{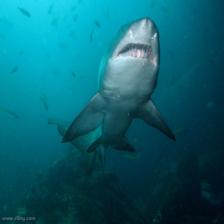} & \includegraphics[width=0.20\textwidth]{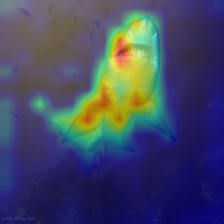} &
  \includegraphics[width=0.20\textwidth]{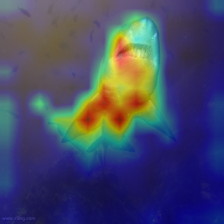} &
\includegraphics[width=0.20\textwidth]{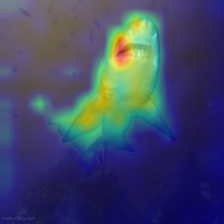}
 \\
  \includegraphics[width=0.20\textwidth]{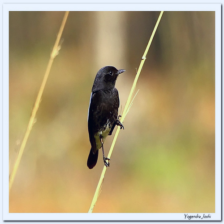} & \includegraphics[width=0.20\textwidth]{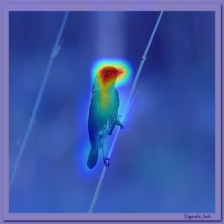} &
  \includegraphics[width=0.20\textwidth]{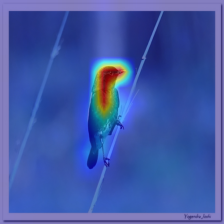} &
\includegraphics[width=0.20\textwidth]{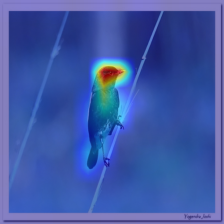}
 \\
  \includegraphics[width=0.20\textwidth]{figures/vision/img4/img_155.png} & \includegraphics[width=0.20\textwidth]{figures/vision/img4/155_basic_custom_lrp.png} &
  \includegraphics[width=0.20\textwidth]{figures/vision/img4/155_basic_custom_lrp_PE_ONLY.png} &
\includegraphics[width=0.20\textwidth]{figures/vision/img4/155_basic_custom_lrp_SEMANTIC_ONLY.png}
 \\
  \includegraphics[width=0.20\textwidth]{figures/vision/img5/img_165.png} & \includegraphics[width=0.20\textwidth]{figures/vision/img5/165_basic_custom_lrp.png} &
  \includegraphics[width=0.20\textwidth]{figures/vision/img5/165_basic_custom_lrp_PE_ONLY.png} &
\includegraphics[width=0.20\textwidth]{figures/vision/img5/165_basic_custom_lrp_SEMANTIC_ONLY.png}
 \\
  \includegraphics[width=0.20\textwidth]{figures/vision/img6/img_157.png} & \includegraphics[width=0.20\textwidth]{figures/vision/img6/157_basic_custom_lrp.png} &
  \includegraphics[width=0.20\textwidth]{figures/vision/img6/157_basic_custom_lrp_PE_ONLY.png} &
\includegraphics[width=0.20\textwidth]
{figures/vision/img6/157_basic_custom_lrp_SEMANTIC_ONLY.png}
 \end{tabular}
\caption{\textbf{Additional Qualitative Results In Vision.} Results of different explanation methods for DeiT. (a) The input image. (b) PA-LRP (ours), which include PE relevancy attribution. (c) PE only LRP, (d) AttnLRP~\citep{achtibat2024attnlrp}, which does not attribute relevancy to PE.\label{fig:attVisExamples1}}
\end{figure*}

\begin{figure*}
\centering
\begin{tabular}{@{\hskip 0.03in}c@{\hskip 0.03in}c@{\hskip 0.03in}c@{\hskip 0.03in}c@{\hskip 0.03in}} 
 
    \includegraphics[width=0.2\textwidth]{figures/vision/img17/img17.png} & 
    \includegraphics[width=0.2\textwidth]{figures/vision/img17/img17_ours.png} &
    \includegraphics[width=0.2\textwidth]{figures/vision/img17/img17_pos.png} & \includegraphics[width=0.2\textwidth]{figures/vision/img17/img17_sem.png} 
 \\
  \includegraphics[width=0.20\textwidth]{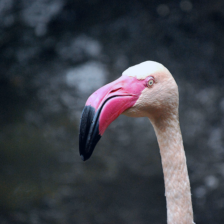} & \includegraphics[width=0.20\textwidth]{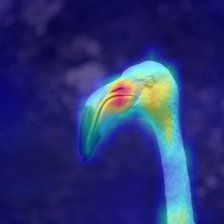} &
  \includegraphics[width=0.20\textwidth]{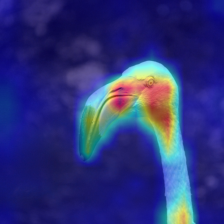} &
\includegraphics[width=0.20\textwidth]{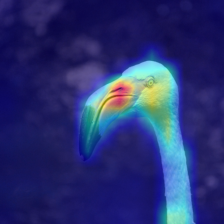}
 \\
  \includegraphics[width=0.20\textwidth]{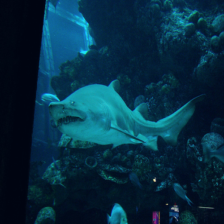} &
\includegraphics[width=0.20\textwidth]{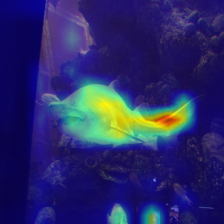} & \includegraphics[width=0.20\textwidth]{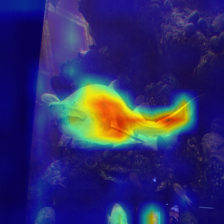} &
  \includegraphics[width=0.20\textwidth]{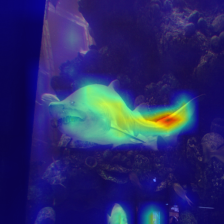}
 \\
\includegraphics[width=0.20\textwidth]{figures/vision/img13/img13.png} & 
\includegraphics[width=0.20\textwidth]{figures/vision/img13/img13_ours.png} &
\includegraphics[width=0.20\textwidth]{figures/vision/img13/img13_pos.png} & \includegraphics[width=0.20\textwidth]{figures/vision/img13/img13_sem.png}
\\
  \includegraphics[width=0.20\textwidth]{figures/vision/img11/img_5.png} & \includegraphics[width=0.20\textwidth]{figures/vision/img11/5_basic_custom_lrp.png} &
  \includegraphics[width=0.20\textwidth]{figures/vision/img11/5_basic_custom_lrp_PE_ONLY.png} &
\includegraphics[width=0.20\textwidth]{figures/vision/img11/5_basic_custom_lrp_SEMANTIC_ONLY.png}
\\
  \includegraphics[width=0.20\textwidth]{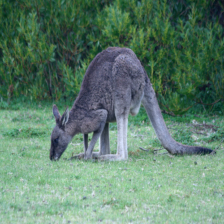} & \includegraphics[width=0.20\textwidth]{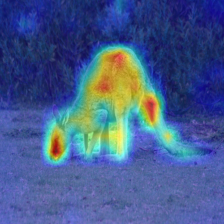} &
  \includegraphics[width=0.20\textwidth]{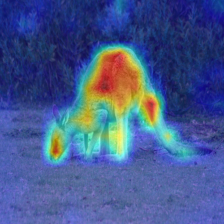} &
   \includegraphics[width=0.20\textwidth]{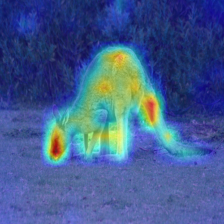}
 \end{tabular}
\caption{\textbf{Additional Qualitative Results In Vision.} Results of different explanation methods for DeiT. (a) The input image. (b) PA-LRP (ours), which include PE relevancy attribution. (c) PE only LRP, (d) AttnLRP~\citep{achtibat2024attnlrp}, which does not attribute relevancy to PE.\label{fig:attVisExamples2}}
\end{figure*}

\begin{figure*}
\centering
\begin{tabular}{@{\hskip 0.03in}c@{\hskip 0.03in}c@{\hskip 0.03in}c@{\hskip 0.03in}c@{\hskip 0.03in}} 
    \includegraphics[width=0.20\textwidth]{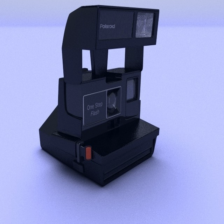} & 
    \includegraphics[width=0.20\textwidth]{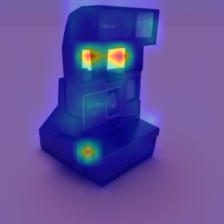} &
    \includegraphics[width=0.20\textwidth]{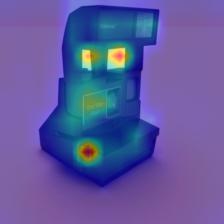} & \includegraphics[width=0.20\textwidth]{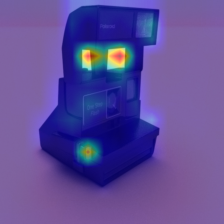}
 \\
   
    \includegraphics[width=0.20\textwidth]{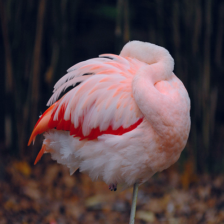} & \includegraphics[width=0.20\textwidth]{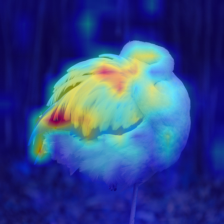} &
  \includegraphics[width=0.20\textwidth]{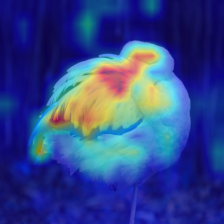} &
\includegraphics[width=0.20\textwidth]{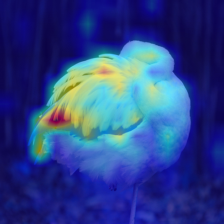}
 \\
    \includegraphics[width=0.20\textwidth]{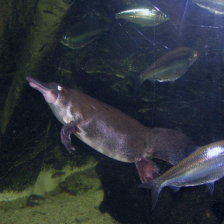} & 
    \includegraphics[width=0.20\textwidth]{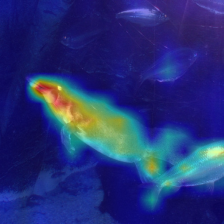} &
    \includegraphics[width=0.20\textwidth]{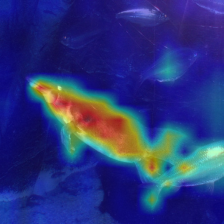} & \includegraphics[width=0.20\textwidth]{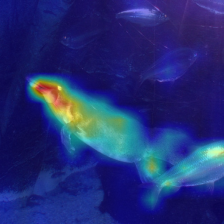}
 \\
    \includegraphics[width=0.20\textwidth]{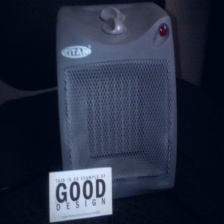} & 
    \includegraphics[width=0.20\textwidth]{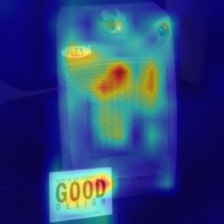} &
    \includegraphics[width=0.20\textwidth]{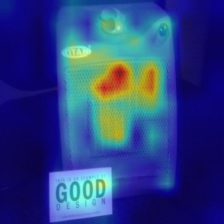} & \includegraphics[width=0.20\textwidth]{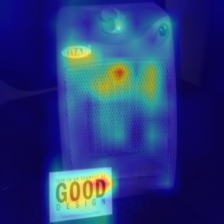}
\\
    \includegraphics[width=0.20\textwidth]{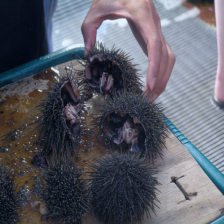} & 
    \includegraphics[width=0.20\textwidth]{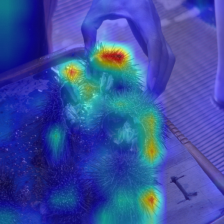} &
    \includegraphics[width=0.20\textwidth]{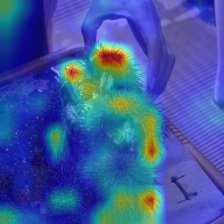} & \includegraphics[width=0.20\textwidth]{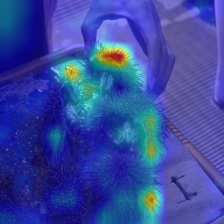}
\\
  \includegraphics[width=0.20\textwidth]{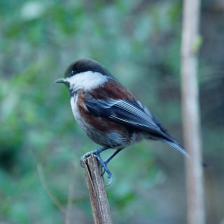} & 
    \includegraphics[width=0.20\textwidth]{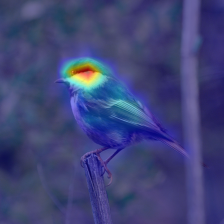} &
    \includegraphics[width=0.20\textwidth]{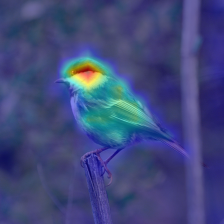} & \includegraphics[width=0.20\textwidth]{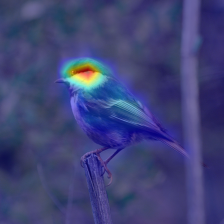}
 \end{tabular}
\caption{\textbf{Additional Qualitative Results In Vision.} Results of different explanation methods for DeiT. (a) The input image. (b) PA-LRP (ours), which include PE relevancy attribution. (c) PE only LRP, (d) AttnLRP~\citep{achtibat2024attnlrp}, which does not attribute relevancy to PE.\label{fig:attVisExamples3}}
\end{figure*}

\clearpage
\section{Additional Qualitative Results in NLP\label{sec:additonalNLPExamples}}

We present qualitative results for NLP in Figure~\ref{fig:qualitative_cmp_NLP}. It can be seen that our method demonstrates better results in highlighting tokens crucial for prediction, along with their surrounding context, emphasizing it's superiority to draw relevancy based on both semantics and positionally. In (b) we see that the amount of artifacts is reduced drastically, with more relevancy channeled to the tokens essential for prediction ("They should have been giving a tribute to Branagh for bringing us one of the greatest films of all time").

\input{NLPSamples.tex}

\newpage

\clearpage
\section{Proofs of Lemmas\label{app:Proofs}}
{
\renewcommand{\thesection}{\arabic{section}} 
\setcounter{lemma}{0}
\setcounter{section}{3}
\begin{lemma}\label{lem:lemma1Appendix}
\textit{For input-level PE transformers, the conservation property is violated when disregarding the positional embeddings' relevancy scores.}
\end{lemma}
}

\begin{proof}[Proof of Lemma~\ref{lem:lemma1Appendix}]
Let $Z$ be our input representation to the first transformer layer, such that $Z = P + E$, where $P$ and $E$ are the token and positional embeddings, respectively. Let $L$ be the number of layers in our transformer. Following the conservation property, the sum of the relevancy scores at any given layer $l$ should uphold: \begin{equation}\sum \mathcal{R}^{(L)} =\sum \mathcal{R}^{(l)} =\sum \mathcal{R}^{(0)} = \sum  R_{Z} = \sum  (\mathcal{R}_{E} + \mathcal{R}_{P})\end{equation}
When ignoring $\mathcal{R}_{P}$, we get the final relevancy attribution map $\mathcal{R}_{input}$, such that: \begin{equation} \sum \mathcal{R}^{(l)} = \sum (\mathcal{R}_{E} + \mathcal{R}_{P}) \neq \mathcal{R}_{E} =  \mathcal{R}_{input} \end{equation}
directly violating the conservation property rule
\end{proof}

{
\renewcommand{\thesection}{\arabic{section}} 
\setcounter{lemma}{1}
\setcounter{section}{3}
\begin{lemma}\label{lem:lemma2Appendix}
\textit{For attention‐level PE transformers, our PE‐LRP rules satisfy the conservation property.}
\end{lemma}
}

\begin{proof}[Proof of Lemma~\ref{lem:lemma2Appendix}]
Let $M$ be the number of layers in our Transformer, and $L$ the sequence length. We denote $\mathcal{R}^{(l)}$ as the relevancy score of the output at layer $l$. Beginning with $\mathcal{R}^{(M)}$ as the the model's output propagating relevancy backwards to achieve the final explanation map for the input embeddings $R_{E}$,  we assume that the standard LRP method does not violate conservation, i.e:
\begin{equation} \forall l \in [M]: \quad   \mathcal{R}^{(M)} = \mathcal{R}^{(l)} = \mathcal{R}_{E} \end{equation}
Recall that for our PA-LRP formulation, we achieve the final explanation map by summing together the semantic attribution $\mathcal{R}_{E}$, achieved by the standard LRP rules, and the positional relevancy $\mathcal{R}_{P}^{(l)}$ distributed across the absorbing sinks at each attention layer $l \in [M]$, giving us the final relevancy map $R_{E} + \sum_{l}\mathcal{R}_{P}^{(l)}$. We aim to prove the following:
\begin{equation}\label{eq:termToProveLemma3.2}   \mathcal{R}^{(M)} =\mathcal{R}_{E} + \sum_{l}\mathcal{R}_{P}^{(l)} \end{equation}
Each attention layer in the transformer is computed using rotary attention:
{\vspace{-2pt}
\begin{equation} 
\forall i \in [L], k \in [K]: \tilde{\bm{Q}}_i = R_{i,k} \bm{Q}_i, \quad \tilde{\bm{K}}_i = R_{i,k} \bm{K}_i,\quad
    \text{RoPE Attention}(X) = \text{Softmax} \left( \frac{\tilde{\bm{Q}}\tilde{\bm{K}}^T}{\sqrt{d_k}} \right) V\,.
\end{equation}
Notice that any computation in this layer which involves more than one tensor, is a matrix multiplication function. Adopting the existing baseline, we use the uniform relevance propagation rule,  distributing the relevancy evenly between components. Thus, the relevancy scores of $Q,K,V,P$, with $P$ denoting the rotation matrix, is equal, and added together to the relevancy of the attention layer's output. The absorbing sink mechanism results in the following:
\begin{equation} \mathcal{R}^{(0)}= \mathcal{R}_{E},\quad \forall l \in [M]: \quad   \mathcal{R}^{(l)} = \mathcal{R}^{(l-1)} + \mathcal{R}_{P}^{(l)} \end{equation}
Following this recursion we would get the exact same result as Eq.~\ref{eq:termToProveLemma3.2}
}
\end{proof}

{
\renewcommand{\thesection}{\arabic{section}} 
\setcounter{lemma}{2}
\setcounter{section}{3}
\begin{lemma}\label{lem:lemma3Appendix}
\textit{For attention‐level PE transformers, current LRP attribution rules achieve low faithfulness, especially when considering positional features.}
\end{lemma}
}

\begin{proof}[Proof of Lemma~\ref{lem:lemma3Appendix}]

\noindent We define a basic learning problem which relies solely on positional features, proving that existing LRP-based explanation methods which don't propagate relevance through positional encodings, will not produce faithful explanations. 
Let us assume we use an auto-regressive transformer model (e.g GPT), with a single causal self-attention with Alibi PE, and the Value projection replaced by an affine transformation (instead of a linear layer). Also, for brevity, let us consider scalar input tokens with sequence length of $L=2$, denoted by $x_{1}, x_{2}$. The final model uses the following keys ($K$), queries ($Q$), and values ($V$):
\begin{equation}
    \forall i \in [1,2]: \quad Q_i = W_Q X_i,\quad ,K_i  =  W_K X_i,\quad V_i = W_V X_i + b
\end{equation}

We apply the Alibi self-attention mechanism, and obtain the final output $O = (O_1, O_2)$:

\begin{equation}
    A(i,j) = \frac{Q_{i}K_{j}}{\sqrt{d}} +m_{h}(i-j),\quad m_{h} = 1, \quad%
    O_2 = A_{2,1} V_1 + A_{2,2} V_2
\end{equation}

To prevent the semantic representation from affecting the prediction, an optimal solution to this problem will assign zeros to $W_{Q}, W_{K}$, namely: $Q=K= \begin{pmatrix} 
0 \\ 0
\end{pmatrix}$. For the Value projection, we assume:  $V=W_{v}+b, \text{with } W_{v} = \begin{pmatrix} 
0 \\ 0 
\end{pmatrix}, b\neq\begin{pmatrix} 
0 \\ 0
\end{pmatrix}$. 

\myparagraph{Relevance propagation.}
Following our settings, we get: 
\begin{equation}
 A = \begin{pmatrix} 
0 & 0 \\ 1 & 0
\end{pmatrix},\text{ giving us:}\quad  Attention(Q,K,V)= A \times V =  \begin{pmatrix} 
0 & 0 \\ 1 & 0
\end{pmatrix}b
\end{equation}
For the backwards relevancy propagation, the relevancy scores of $Attention$ are distrusted between $A$ and $V$ based on the standard $\text{Gradient} \times \text{Input}$. Regardless, we now consider how relevancy scores of both terms $\mathcal{R}_{V},\mathcal{R}_{score}$ are propagated back to the input ${x}$. 
\begin{itemize}
 \item \textbf{$\mathcal{R}_{V}\rightarrow \mathcal{R}_{x}.\quad$} recall that $W_{v}$ are assigned with zeros. Given that the fundamental $\epsilon$-LRP rule for affine transformations ignores the bias term completely and uses the weights $W_{v}$ as a measure of weighting the relevancy scores, we get that zero relevancy scores are produced for both tokens.
\item \textbf{$\mathcal{R}_{A}\rightarrow \mathcal{R}_{x}.\quad$} Following the standard LRP rules, the positional terms would be considered a constant, and therefor, 100\% of the distribution would be directed to the queries and keys. Given that $W_{Q}, W_{K}$ are assigned with zeros, we again get zero relevancy scores being propagated to $x$. 
 \end{itemize}
 Given that the relevancy scores propagated back from the attention layer are all assigned with zeros, we will get a final attribution map of zeroes, indicating the same level of impact for all tokens. This of course, yields an unfaithful explanation. In contrast, our method makes positional terms attributable, maintaining relevancy scores that would otherwise be zeroed out due to existing rules.
\end{proof}

\clearpage

\renewcommand{\thesection}{\Alph{section}} 
\setcounter{section}{5}

\section{Conservation Percentage Results \label{app:conservationPercentageResults}}
{
We measure the sum of relevance for DeiT model at different capacities: Tiny, Small, and Base. The figure provides clear visualization for the violation of the conservation property, with PE relevancy constituting 16.75\%, 22.39\%, and 9.22\% out of the total relevancy for Tiny, Small, and Base DeiT models, respectively.

\begin{figure}[h]
    \centering  
    \includegraphics[width=0.5\linewidth]{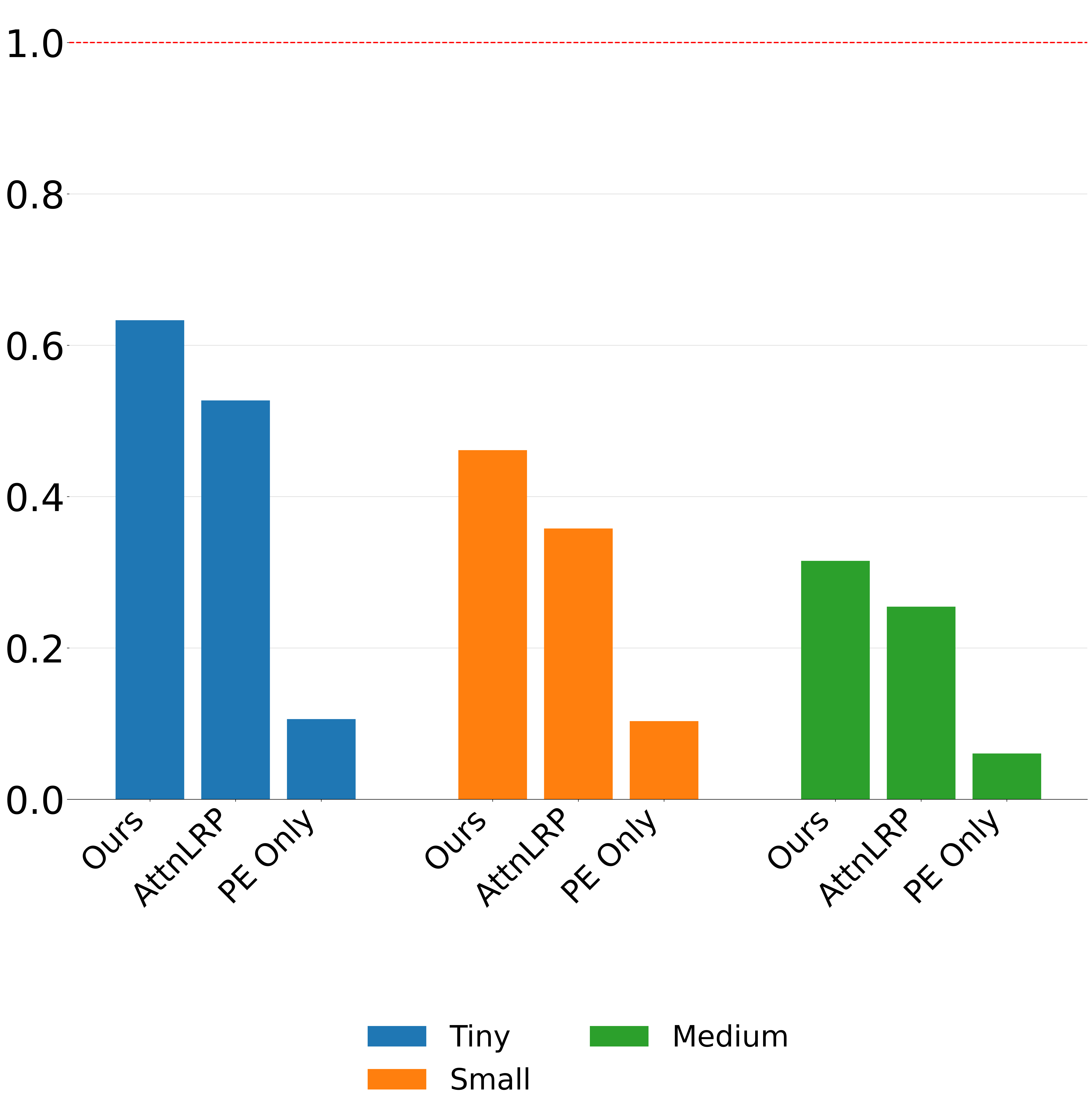}
    \vspace{-4pt}
    \caption{\small We assess both the positional relevance and the non-positional relevance for DeiT models at different capacities, visualizing the violation of conservation rule, with high non-negligible ratio between the entire relevance in the models 'ours' and positional-associated relevance 'PE Only'.}
    \label{fig:conservationVisFig}
\end{figure}




\end{document}

%% file: NLPSamples.tex
\begin{figure*}[t]
\centering \small
\resizebox{0.89\textwidth}{!}{%
    \begin{tabular}{l}
    \input{figures/NLP/samples/3/heatmap_11_baseline}\\
    \input{figures/NLP/samples/3/heatmap_11_PE}\\ 
    \input{figures/NLP/samples/3/heatmap_11_LRP+PE+REFORM}\\
    \multicolumn{1}{c}{(a)} \\ 
    \input{figures/NLP/samples/7/heatmap_25_baseline}\\
    \input{figures/NLP/samples/7/heatmap_25_PE} \\ 
    \input{figures/NLP/samples/7/heatmap_25_PE+LRP}\\
    \multicolumn{1}{c}{(b)} \\ 
     \end{tabular}
}
 \vspace{-5pt}
\caption{\small \textbf{Qualitative Results in NLP.} Both groups (a) and (b) present results from different explanation methods for the same example obtained from the IMDB benchmark. In each group, the first row represents the AttnLRP baseline, followed by the PE-only variant in the middle, and finally, our maps at the end.\label{fig:qualitative_cmp_NLP}}
\end{figure*}




%

\newpage

%% file: figures/NLP/samples/3/heatmap_11_baseline.tex
  
    \fbox{
    \parbox{\textwidth}{
    \setlength\fboxsep{0pt}
     \colorbox[RGB]{255,0,0}{\strut <s>} \colorbox[RGB]{255,226,226}{\strut  Great} \colorbox[RGB]{255,248,248}{\strut  just} \colorbox[RGB]{255,234,234}{\strut  great} \colorbox[RGB]{255,222,222}{\strut !} \colorbox[RGB]{246,246,255}{\strut  The} \colorbox[RGB]{131,131,255}{\strut  West} \colorbox[RGB]{255,252,252}{\strut  Coast} \colorbox[RGB]{252,252,255}{\strut  got} \colorbox[RGB]{255,254,254}{\strut  "} \colorbox[RGB]{255,254,254}{\strut Dir} \colorbox[RGB]{255,254,254}{\strut ty} \colorbox[RGB]{254,254,255}{\strut "} \colorbox[RGB]{243,243,255}{\strut  Harry} \colorbox[RGB]{255,254,254}{\strut  Cal} \colorbox[RGB]{255,252,252}{\strut la} \colorbox[RGB]{255,250,250}{\strut han} \colorbox[RGB]{255,254,254}{\strut ,} \colorbox[RGB]{254,254,255}{\strut  the} \colorbox[RGB]{250,250,255}{\strut  East} \colorbox[RGB]{255,254,254}{\strut  Coast} \colorbox[RGB]{254,254,255}{\strut  got} \colorbox[RGB]{255,242,242}{\strut  Sh} \colorbox[RGB]{255,254,254}{\strut ark} \colorbox[RGB]{255,252,252}{\strut y} \colorbox[RGB]{255,252,252}{\strut .} \colorbox[RGB]{255,252,252}{\strut  B} \colorbox[RGB]{255,254,254}{\strut urt} \colorbox[RGB]{227,227,255}{\strut  Reyn} \colorbox[RGB]{246,246,255}{\strut olds} \colorbox[RGB]{250,250,255}{\strut  plays} \colorbox[RGB]{255,250,250}{\strut  Sh} \colorbox[RGB]{255,254,254}{\strut ark} \colorbox[RGB]{255,252,252}{\strut y} \colorbox[RGB]{255,254,254}{\strut  in} \colorbox[RGB]{188,188,255}{\strut  "} \colorbox[RGB]{255,250,250}{\strut Sh} \colorbox[RGB]{254,254,255}{\strut ark} \colorbox[RGB]{255,254,254}{\strut y} \colorbox[RGB]{255,244,244}{\strut '} \colorbox[RGB]{255,254,254}{\strut s} \colorbox[RGB]{255,248,248}{\strut  Machine} \colorbox[RGB]{252,252,255}{\strut "} \colorbox[RGB]{255,248,248}{\strut  and} \colorbox[RGB]{255,242,242}{\strut  I} \colorbox[RGB]{255,240,240}{\strut  enjoyed} \colorbox[RGB]{255,226,226}{\strut  every} \colorbox[RGB]{255,250,250}{\strut  minute} \colorbox[RGB]{255,254,254}{\strut  of} \colorbox[RGB]{255,248,248}{\strut  it} \colorbox[RGB]{255,226,226}{\strut .} \colorbox[RGB]{255,246,246}{\strut  Play} \colorbox[RGB]{255,254,254}{\strut ing} \colorbox[RGB]{255,254,254}{\strut  a} \colorbox[RGB]{246,246,255}{\strut  ma} \colorbox[RGB]{255,254,254}{\strut ver} \colorbox[RGB]{255,254,254}{\strut ick} \colorbox[RGB]{255,254,254}{\strut  n} \colorbox[RGB]{255,254,254}{\strut arc} \colorbox[RGB]{255,254,254}{\strut ot} \colorbox[RGB]{254,254,255}{\strut ics} \colorbox[RGB]{255,252,252}{\strut  cop} \colorbox[RGB]{255,250,250}{\strut  in} \colorbox[RGB]{255,248,248}{\strut  Atlanta} \colorbox[RGB]{255,248,248}{\strut ,} \colorbox[RGB]{255,252,252}{\strut  G} \colorbox[RGB]{255,250,250}{\strut A} \colorbox[RGB]{252,252,255}{\strut  is} \colorbox[RGB]{255,252,252}{\strut  just} \colorbox[RGB]{254,254,255}{\strut  what} \colorbox[RGB]{255,250,250}{\strut  everyone} \colorbox[RGB]{255,252,252}{\strut  wants} \colorbox[RGB]{255,244,244}{\strut .} \colorbox[RGB]{243,243,255}{\strut  Instead} \colorbox[RGB]{255,254,254}{\strut  of} \colorbox[RGB]{255,254,254}{\strut  susp} \colorbox[RGB]{255,254,254}{\strut ension} \colorbox[RGB]{255,252,252}{\strut ,} \colorbox[RGB]{255,250,250}{\strut  he} \colorbox[RGB]{248,248,255}{\strut '} \colorbox[RGB]{255,254,254}{\strut s} \colorbox[RGB]{254,254,255}{\strut  sent} \colorbox[RGB]{255,254,254}{\strut  to} \colorbox[RGB]{255,250,250}{\strut  vice} \colorbox[RGB]{254,254,255}{\strut  squad} \colorbox[RGB]{254,254,255}{\strut .} \colorbox[RGB]{242,242,255}{\strut  Like} \colorbox[RGB]{255,254,254}{\strut  in} \colorbox[RGB]{255,252,252}{\strut  the} \colorbox[RGB]{255,254,254}{\strut  D} \colorbox[RGB]{254,254,255}{\strut irty} \colorbox[RGB]{252,252,255}{\strut  Harry} \colorbox[RGB]{254,254,255}{\strut  mov} \colorbox[RGB]{255,254,254}{\strut ies} \colorbox[RGB]{250,250,255}{\strut  or} \colorbox[RGB]{255,254,254}{\strut  any} \colorbox[RGB]{255,254,254}{\strut  other} \colorbox[RGB]{255,246,246}{\strut  cop} \colorbox[RGB]{254,254,255}{\strut  mov} \colorbox[RGB]{255,254,254}{\strut ies} \colorbox[RGB]{255,252,252}{\strut ,} \colorbox[RGB]{255,254,254}{\strut  the} \colorbox[RGB]{255,252,252}{\strut  captain} \colorbox[RGB]{254,254,255}{\strut  is} \colorbox[RGB]{254,254,255}{\strut  always} \colorbox[RGB]{254,254,255}{\strut  going} \colorbox[RGB]{255,254,254}{\strut  to} \colorbox[RGB]{254,254,255}{\strut  be} \colorbox[RGB]{255,254,254}{\strut  the} \colorbox[RGB]{252,252,255}{\strut  j} \colorbox[RGB]{248,248,255}{\strut erk} \colorbox[RGB]{246,246,255}{\strut .} \colorbox[RGB]{255,254,254}{\strut  When} \colorbox[RGB]{255,254,254}{\strut  I} \colorbox[RGB]{255,254,254}{\strut  was} \colorbox[RGB]{255,254,254}{\strut  a} \colorbox[RGB]{255,254,254}{\strut  kid} \colorbox[RGB]{255,254,254}{\strut ,} \colorbox[RGB]{255,254,254}{\strut  I} \colorbox[RGB]{255,254,254}{\strut  was} \colorbox[RGB]{255,254,254}{\strut  curious} \colorbox[RGB]{255,252,252}{\strut  what} \colorbox[RGB]{255,252,252}{\strut  that} \colorbox[RGB]{252,252,255}{\strut  movie} \colorbox[RGB]{252,252,255}{\strut  meant} \colorbox[RGB]{252,252,255}{\strut  "} \colorbox[RGB]{255,252,252}{\strut Sh} \colorbox[RGB]{254,254,255}{\strut ark} \colorbox[RGB]{255,254,254}{\strut y} \colorbox[RGB]{255,250,250}{\strut '} \colorbox[RGB]{255,254,254}{\strut s} \colorbox[RGB]{255,254,254}{\strut  Machine} \colorbox[RGB]{255,254,254}{\strut ".} \colorbox[RGB]{255,248,248}{\strut  Well} \colorbox[RGB]{254,254,255}{\strut  I} \colorbox[RGB]{255,252,252}{\strut  knew} \colorbox[RGB]{254,254,255}{\strut  who} \colorbox[RGB]{254,254,255}{\strut  played} \colorbox[RGB]{255,254,254}{\strut  Sh} \colorbox[RGB]{255,254,254}{\strut ark} \colorbox[RGB]{255,254,254}{\strut y} \colorbox[RGB]{255,254,254}{\strut ,} \colorbox[RGB]{254,254,255}{\strut  I} \colorbox[RGB]{254,254,255}{\strut  wonder} \colorbox[RGB]{254,254,255}{\strut  what} \colorbox[RGB]{254,254,255}{\strut  his} \colorbox[RGB]{255,254,254}{\strut  machine} \colorbox[RGB]{254,254,255}{\strut  was} \colorbox[RGB]{252,252,255}{\strut .} \colorbox[RGB]{255,254,254}{\strut  It} \colorbox[RGB]{254,254,255}{\strut  was} \colorbox[RGB]{255,254,254}{\strut  his} \colorbox[RGB]{254,254,255}{\strut  GROUP} \colorbox[RGB]{255,252,252}{\strut  of} \colorbox[RGB]{255,252,252}{\strut  fellow} \colorbox[RGB]{255,254,254}{\strut  c} \colorbox[RGB]{255,252,252}{\strut ops} \colorbox[RGB]{255,254,254}{\strut .} \colorbox[RGB]{248,248,255}{\strut  After} \colorbox[RGB]{255,254,254}{\strut  un} \colorbox[RGB]{255,254,254}{\strut cover} \colorbox[RGB]{255,254,254}{\strut ing} \colorbox[RGB]{255,254,254}{\strut  the} \colorbox[RGB]{255,250,250}{\strut  murder} \colorbox[RGB]{255,252,252}{\strut ,} \colorbox[RGB]{255,252,252}{\strut  he} \colorbox[RGB]{254,254,255}{\strut  goes} \colorbox[RGB]{255,254,254}{\strut  all} \colorbox[RGB]{255,252,252}{\strut  out} \colorbox[RGB]{255,254,254}{\strut  to} \colorbox[RGB]{255,254,254}{\strut  find} \colorbox[RGB]{255,254,254}{\strut  the} \colorbox[RGB]{254,254,255}{\strut  per} \colorbox[RGB]{254,254,255}{\strut p} \colorbox[RGB]{254,254,255}{\strut .} \colorbox[RGB]{252,252,255}{\strut  When} \colorbox[RGB]{254,254,255}{\strut  it} \colorbox[RGB]{250,250,255}{\strut  turns} \colorbox[RGB]{254,254,255}{\strut  out} \colorbox[RGB]{255,254,254}{\strut  to} \colorbox[RGB]{255,254,254}{\strut  be} \colorbox[RGB]{255,254,254}{\strut  a} \colorbox[RGB]{255,252,252}{\strut  big} \colorbox[RGB]{255,252,252}{\strut  time} \colorbox[RGB]{255,248,248}{\strut  mob} \colorbox[RGB]{255,254,254}{\strut  b} \colorbox[RGB]{254,254,255}{\strut oss} \colorbox[RGB]{255,254,254}{\strut ,} \colorbox[RGB]{255,252,252}{\strut  Sh} \colorbox[RGB]{255,254,254}{\strut ark} \colorbox[RGB]{255,254,254}{\strut y} \colorbox[RGB]{252,252,255}{\strut  doesn} \colorbox[RGB]{240,240,255}{\strut '} \colorbox[RGB]{255,254,254}{\strut t} \colorbox[RGB]{255,254,254}{\strut  play} \colorbox[RGB]{255,254,254}{\strut  around} \colorbox[RGB]{255,254,254}{\strut .} \colorbox[RGB]{254,254,255}{\strut  When} \colorbox[RGB]{255,254,254}{\strut  he} \colorbox[RGB]{254,254,255}{\strut  gets} \colorbox[RGB]{255,254,254}{\strut  the} \colorbox[RGB]{254,254,255}{\strut  other} \colorbox[RGB]{248,248,255}{\strut  prost} \colorbox[RGB]{255,254,254}{\strut itute} \colorbox[RGB]{255,252,252}{\strut  into} \colorbox[RGB]{255,252,252}{\strut  safety} \colorbox[RGB]{255,252,252}{\strut ,} \colorbox[RGB]{255,254,254}{\strut  Sh} \colorbox[RGB]{255,254,254}{\strut ark} \colorbox[RGB]{255,254,254}{\strut y} \colorbox[RGB]{255,254,254}{\strut  f} \colorbox[RGB]{255,254,254}{\strut ights} \colorbox[RGB]{255,252,252}{\strut  back} \colorbox[RGB]{255,252,252}{\strut  hard} \colorbox[RGB]{255,252,252}{\strut  and} \colorbox[RGB]{255,254,254}{\strut  good} \colorbox[RGB]{246,246,255}{\strut  despite} \colorbox[RGB]{254,254,255}{\strut  losing} \colorbox[RGB]{254,254,255}{\strut  a} \colorbox[RGB]{250,250,255}{\strut  finger} \colorbox[RGB]{255,254,254}{\strut  to} \colorbox[RGB]{255,254,254}{\strut  the} \colorbox[RGB]{254,254,255}{\strut  th} \colorbox[RGB]{254,254,255}{\strut ug} \colorbox[RGB]{252,252,255}{\strut .} \colorbox[RGB]{242,242,255}{\strut  And} \colorbox[RGB]{255,254,254}{\strut  I} \colorbox[RGB]{255,244,244}{\strut  also} \colorbox[RGB]{248,248,255}{\strut  like} \colorbox[RGB]{254,254,255}{\strut  the} \colorbox[RGB]{255,254,254}{\strut  part} \colorbox[RGB]{255,254,254}{\strut  where} \colorbox[RGB]{255,254,254}{\strut  the} \colorbox[RGB]{252,252,255}{\strut  bad} \colorbox[RGB]{252,252,255}{\strut  gets} \colorbox[RGB]{252,252,255}{\strut  blow} \colorbox[RGB]{254,254,255}{\strut n} \colorbox[RGB]{255,254,254}{\strut  out} \colorbox[RGB]{255,254,254}{\strut  of} \colorbox[RGB]{255,254,254}{\strut  the} \colorbox[RGB]{254,254,255}{\strut  building} \colorbox[RGB]{255,252,252}{\strut  through} \colorbox[RGB]{254,254,255}{\strut  a} \colorbox[RGB]{254,254,255}{\strut  plate} \colorbox[RGB]{254,254,255}{\strut  glass} \colorbox[RGB]{254,254,255}{\strut  window} \colorbox[RGB]{242,242,255}{\strut .} \colorbox[RGB]{250,250,255}{\strut  That} \colorbox[RGB]{248,248,255}{\strut  was} \colorbox[RGB]{255,254,254}{\strut  the} \colorbox[RGB]{255,254,254}{\strut  B} \colorbox[RGB]{254,254,255}{\strut OM} \colorbox[RGB]{254,254,255}{\strut B} \colorbox[RGB]{255,246,246}{\strut !} \colorbox[RGB]{255,236,236}{\strut  R} \colorbox[RGB]{255,238,238}{\strut andy} \colorbox[RGB]{254,254,255}{\strut  C} \colorbox[RGB]{255,254,254}{\strut raw} \colorbox[RGB]{255,246,246}{\strut ford} \colorbox[RGB]{240,240,255}{\strut '} \colorbox[RGB]{255,250,250}{\strut s} \colorbox[RGB]{246,246,255}{\strut  "} \colorbox[RGB]{255,250,250}{\strut St} \colorbox[RGB]{255,232,232}{\strut reet} \colorbox[RGB]{255,254,254}{\strut  Life} \colorbox[RGB]{255,246,246}{\strut "} \colorbox[RGB]{255,230,230}{\strut  really} \colorbox[RGB]{248,248,255}{\strut  put} \colorbox[RGB]{255,252,252}{\strut  the} \colorbox[RGB]{227,227,255}{\strut  movie} \colorbox[RGB]{254,254,255}{\strut  in} \colorbox[RGB]{255,254,254}{\strut  the} \colorbox[RGB]{255,250,250}{\strut  right} \colorbox[RGB]{252,252,255}{\strut  m} \colorbox[RGB]{250,250,255}{\strut ood} \colorbox[RGB]{255,252,252}{\strut ,} \colorbox[RGB]{255,244,244}{\strut  and} \colorbox[RGB]{254,254,255}{\strut  the} \colorbox[RGB]{250,250,255}{\strut  movie} \colorbox[RGB]{255,244,244}{\strut  itself} \colorbox[RGB]{255,252,252}{\strut  is} \colorbox[RGB]{255,234,234}{\strut  really} \colorbox[RGB]{255,246,246}{\strut  a} \colorbox[RGB]{255,206,206}{\strut  great} \colorbox[RGB]{255,240,240}{\strut  hit} \colorbox[RGB]{255,240,240}{\strut  to} \colorbox[RGB]{255,234,234}{\strut  me} \colorbox[RGB]{255,248,248}{\strut ,} \colorbox[RGB]{254,254,255}{\strut  AL} \colorbox[RGB]{255,244,244}{\strut WA} \colorbox[RGB]{255,226,226}{\strut YS} \colorbox[RGB]{255,208,208}{\strut !} \colorbox[RGB]{250,250,255}{\strut  R} \colorbox[RGB]{255,214,214}{\strut ating} \colorbox[RGB]{238,238,255}{\strut  } \colorbox[RGB]{255,214,214}{\strut 4} \colorbox[RGB]{255,242,242}{\strut  out} \colorbox[RGB]{255,238,238}{\strut  of} \colorbox[RGB]{255,234,234}{\strut  } \colorbox[RGB]{255,162,162}{\strut 5} \colorbox[RGB]{255,178,178}{\strut  stars} \colorbox[RGB]{255,33,33}{\strut .}}}
     

%% file: figures/NLP/samples/3/heatmap_11_PE.tex
    \fbox{
    \parbox{\textwidth}{
    \setlength\fboxsep{0pt}
     \colorbox[RGB]{255,0,0}{\strut <s>} \colorbox[RGB]{255,226,226}{\strut  Great} \colorbox[RGB]{255,246,246}{\strut  just} \colorbox[RGB]{255,234,234}{\strut  great} \colorbox[RGB]{255,226,226}{\strut !} \colorbox[RGB]{255,252,252}{\strut  The} \colorbox[RGB]{240,240,255}{\strut  West} \colorbox[RGB]{255,250,250}{\strut  Coast} \colorbox[RGB]{255,252,252}{\strut  got} \colorbox[RGB]{255,252,252}{\strut  "} \colorbox[RGB]{255,252,252}{\strut Dir} \colorbox[RGB]{255,252,252}{\strut ty} \colorbox[RGB]{255,252,252}{\strut "} \colorbox[RGB]{255,254,254}{\strut  Harry} \colorbox[RGB]{255,250,250}{\strut  Cal} \colorbox[RGB]{255,254,254}{\strut la} \colorbox[RGB]{255,246,246}{\strut han} \colorbox[RGB]{255,250,250}{\strut ,} \colorbox[RGB]{255,252,252}{\strut  the} \colorbox[RGB]{255,254,254}{\strut  East} \colorbox[RGB]{255,252,252}{\strut  Coast} \colorbox[RGB]{255,252,252}{\strut  got} \colorbox[RGB]{255,250,250}{\strut  Sh} \colorbox[RGB]{255,252,252}{\strut ark} \colorbox[RGB]{255,250,250}{\strut y} \colorbox[RGB]{255,244,244}{\strut .} \colorbox[RGB]{255,250,250}{\strut  B} \colorbox[RGB]{255,250,250}{\strut urt} \colorbox[RGB]{254,254,255}{\strut  Reyn} \colorbox[RGB]{255,252,252}{\strut olds} \colorbox[RGB]{255,250,250}{\strut  plays} \colorbox[RGB]{255,252,252}{\strut  Sh} \colorbox[RGB]{255,254,254}{\strut ark} \colorbox[RGB]{255,254,254}{\strut y} \colorbox[RGB]{255,250,250}{\strut  in} \colorbox[RGB]{248,248,255}{\strut  "} \colorbox[RGB]{255,252,252}{\strut Sh} \colorbox[RGB]{255,254,254}{\strut ark} \colorbox[RGB]{255,252,252}{\strut y} \colorbox[RGB]{255,248,248}{\strut '} \colorbox[RGB]{255,254,254}{\strut s} \colorbox[RGB]{255,250,250}{\strut  Machine} \colorbox[RGB]{255,252,252}{\strut "} \colorbox[RGB]{255,250,250}{\strut  and} \colorbox[RGB]{255,246,246}{\strut  I} \colorbox[RGB]{255,240,240}{\strut  enjoyed} \colorbox[RGB]{255,236,236}{\strut  every} \colorbox[RGB]{255,248,248}{\strut  minute} \colorbox[RGB]{255,252,252}{\strut  of} \colorbox[RGB]{255,248,248}{\strut  it} \colorbox[RGB]{255,230,230}{\strut .} \colorbox[RGB]{255,246,246}{\strut  Play} \colorbox[RGB]{255,252,252}{\strut ing} \colorbox[RGB]{255,252,252}{\strut  a} \colorbox[RGB]{254,254,255}{\strut  ma} \colorbox[RGB]{255,254,254}{\strut ver} \colorbox[RGB]{255,254,254}{\strut ick} \colorbox[RGB]{255,254,254}{\strut  n} \colorbox[RGB]{255,254,254}{\strut arc} \colorbox[RGB]{255,252,252}{\strut ot} \colorbox[RGB]{255,254,254}{\strut ics} \colorbox[RGB]{255,250,250}{\strut  cop} \colorbox[RGB]{255,250,250}{\strut  in} \colorbox[RGB]{255,248,248}{\strut  Atlanta} \colorbox[RGB]{255,248,248}{\strut ,} \colorbox[RGB]{255,250,250}{\strut  G} \colorbox[RGB]{255,248,248}{\strut A} \colorbox[RGB]{255,252,252}{\strut  is} \colorbox[RGB]{255,252,252}{\strut  just} \colorbox[RGB]{255,250,250}{\strut  what} \colorbox[RGB]{255,250,250}{\strut  everyone} \colorbox[RGB]{255,252,252}{\strut  wants} \colorbox[RGB]{255,242,242}{\strut .} \colorbox[RGB]{254,254,255}{\strut  Instead} \colorbox[RGB]{255,252,252}{\strut  of} \colorbox[RGB]{255,254,254}{\strut  susp} \colorbox[RGB]{255,254,254}{\strut ension} \colorbox[RGB]{255,250,250}{\strut ,} \colorbox[RGB]{255,250,250}{\strut  he} \colorbox[RGB]{255,250,250}{\strut '} \colorbox[RGB]{255,254,254}{\strut s} \colorbox[RGB]{255,254,254}{\strut  sent} \colorbox[RGB]{255,254,254}{\strut  to} \colorbox[RGB]{255,252,252}{\strut  vice} \colorbox[RGB]{255,254,254}{\strut  squad} \colorbox[RGB]{255,250,250}{\strut .} \colorbox[RGB]{254,254,255}{\strut  Like} \colorbox[RGB]{255,252,252}{\strut  in} \colorbox[RGB]{255,254,254}{\strut  the} \colorbox[RGB]{255,254,254}{\strut  D} \colorbox[RGB]{255,254,254}{\strut irty} \colorbox[RGB]{255,254,254}{\strut  Harry} \colorbox[RGB]{255,254,254}{\strut  mov} \colorbox[RGB]{255,252,252}{\strut ies} \colorbox[RGB]{255,252,252}{\strut  or} \colorbox[RGB]{255,252,252}{\strut  any} \colorbox[RGB]{255,254,254}{\strut  other} \colorbox[RGB]{255,246,246}{\strut  cop} \colorbox[RGB]{255,254,254}{\strut  mov} \colorbox[RGB]{255,254,254}{\strut ies} \colorbox[RGB]{255,250,250}{\strut ,} \colorbox[RGB]{255,254,254}{\strut  the} \colorbox[RGB]{255,252,252}{\strut  captain} \colorbox[RGB]{255,254,254}{\strut  is} \colorbox[RGB]{255,254,254}{\strut  always} \colorbox[RGB]{255,254,254}{\strut  going} \colorbox[RGB]{255,254,254}{\strut  to} \colorbox[RGB]{255,254,254}{\strut  be} \colorbox[RGB]{255,254,254}{\strut  the} \colorbox[RGB]{255,254,254}{\strut  j} \colorbox[RGB]{255,252,252}{\strut erk} \colorbox[RGB]{255,252,252}{\strut .} \colorbox[RGB]{255,252,252}{\strut  When} \colorbox[RGB]{255,250,250}{\strut  I} \colorbox[RGB]{255,254,254}{\strut  was} \colorbox[RGB]{255,254,254}{\strut  a} \colorbox[RGB]{255,252,252}{\strut  kid} \colorbox[RGB]{255,252,252}{\strut ,} \colorbox[RGB]{255,252,252}{\strut  I} \colorbox[RGB]{255,254,254}{\strut  was} \colorbox[RGB]{255,252,252}{\strut  curious} \colorbox[RGB]{255,252,252}{\strut  what} \colorbox[RGB]{255,252,252}{\strut  that} \colorbox[RGB]{255,252,252}{\strut  movie} \colorbox[RGB]{255,252,252}{\strut  meant} \colorbox[RGB]{255,252,252}{\strut  "} \colorbox[RGB]{255,252,252}{\strut Sh} \colorbox[RGB]{255,254,254}{\strut ark} \colorbox[RGB]{255,254,254}{\strut y} \colorbox[RGB]{255,252,252}{\strut '} \colorbox[RGB]{255,254,254}{\strut s} \colorbox[RGB]{255,252,252}{\strut  Machine} \colorbox[RGB]{255,246,246}{\strut ".} \colorbox[RGB]{255,250,250}{\strut  Well} \colorbox[RGB]{255,252,252}{\strut  I} \colorbox[RGB]{255,252,252}{\strut  knew} \colorbox[RGB]{255,252,252}{\strut  who} \colorbox[RGB]{255,254,254}{\strut  played} \colorbox[RGB]{255,254,254}{\strut  Sh} \colorbox[RGB]{255,254,254}{\strut ark} \colorbox[RGB]{255,254,254}{\strut y} \colorbox[RGB]{255,252,252}{\strut ,} \colorbox[RGB]{255,254,254}{\strut  I} \colorbox[RGB]{255,254,254}{\strut  wonder} \colorbox[RGB]{255,254,254}{\strut  what} \colorbox[RGB]{255,254,254}{\strut  his} \colorbox[RGB]{255,252,252}{\strut  machine} \colorbox[RGB]{255,254,254}{\strut  was} \colorbox[RGB]{255,250,250}{\strut .} \colorbox[RGB]{255,252,252}{\strut  It} \colorbox[RGB]{255,252,252}{\strut  was} \colorbox[RGB]{255,252,252}{\strut  his} \colorbox[RGB]{255,250,250}{\strut  GROUP} \colorbox[RGB]{255,252,252}{\strut  of} \colorbox[RGB]{255,252,252}{\strut  fellow} \colorbox[RGB]{255,254,254}{\strut  c} \colorbox[RGB]{255,250,250}{\strut ops} \colorbox[RGB]{255,250,250}{\strut .} \colorbox[RGB]{255,254,254}{\strut  After} \colorbox[RGB]{255,254,254}{\strut  un} \colorbox[RGB]{255,254,254}{\strut cover} \colorbox[RGB]{255,252,252}{\strut ing} \colorbox[RGB]{255,254,254}{\strut  the} \colorbox[RGB]{255,250,250}{\strut  murder} \colorbox[RGB]{255,252,252}{\strut ,} \colorbox[RGB]{255,252,252}{\strut  he} \colorbox[RGB]{255,254,254}{\strut  goes} \colorbox[RGB]{255,252,252}{\strut  all} \colorbox[RGB]{255,252,252}{\strut  out} \colorbox[RGB]{255,254,254}{\strut  to} \colorbox[RGB]{255,254,254}{\strut  find} \colorbox[RGB]{255,254,254}{\strut  the} \colorbox[RGB]{255,254,254}{\strut  per} \colorbox[RGB]{255,254,254}{\strut p} \colorbox[RGB]{255,248,248}{\strut .} \colorbox[RGB]{255,254,254}{\strut  When} \colorbox[RGB]{255,254,254}{\strut  it} \colorbox[RGB]{255,254,254}{\strut  turns} \colorbox[RGB]{255,252,252}{\strut  out} \colorbox[RGB]{255,254,254}{\strut  to} \colorbox[RGB]{255,254,254}{\strut  be} \colorbox[RGB]{255,254,254}{\strut  a} \colorbox[RGB]{255,252,252}{\strut  big} \colorbox[RGB]{255,252,252}{\strut  time} \colorbox[RGB]{255,250,250}{\strut  mob} \colorbox[RGB]{255,254,254}{\strut  b} \colorbox[RGB]{255,254,254}{\strut oss} \colorbox[RGB]{255,252,252}{\strut ,} \colorbox[RGB]{255,252,252}{\strut  Sh} \colorbox[RGB]{255,254,254}{\strut ark} \colorbox[RGB]{255,254,254}{\strut y} \colorbox[RGB]{255,254,254}{\strut  doesn} \colorbox[RGB]{255,248,248}{\strut '} \colorbox[RGB]{255,254,254}{\strut t} \colorbox[RGB]{255,254,254}{\strut  play} \colorbox[RGB]{255,254,254}{\strut  around} \colorbox[RGB]{255,248,248}{\strut .} \colorbox[RGB]{255,254,254}{\strut  When} \colorbox[RGB]{255,254,254}{\strut  he} \colorbox[RGB]{255,254,254}{\strut  gets} \colorbox[RGB]{255,254,254}{\strut  the} \colorbox[RGB]{255,254,254}{\strut  other} \colorbox[RGB]{254,254,255}{\strut  prost} \colorbox[RGB]{255,252,252}{\strut itute} \colorbox[RGB]{255,252,252}{\strut  into} \colorbox[RGB]{255,252,252}{\strut  safety} \colorbox[RGB]{255,252,252}{\strut ,} \colorbox[RGB]{255,254,254}{\strut  Sh} \colorbox[RGB]{255,254,254}{\strut ark} \colorbox[RGB]{255,254,254}{\strut y} \colorbox[RGB]{255,254,254}{\strut  f} \colorbox[RGB]{255,252,252}{\strut ights} \colorbox[RGB]{255,252,252}{\strut  back} \colorbox[RGB]{255,250,250}{\strut  hard} \colorbox[RGB]{255,252,252}{\strut  and} \colorbox[RGB]{255,252,252}{\strut  good} \colorbox[RGB]{254,254,255}{\strut  despite} \colorbox[RGB]{255,254,254}{\strut  losing} \colorbox[RGB]{255,254,254}{\strut  a} \colorbox[RGB]{255,254,254}{\strut  finger} \colorbox[RGB]{255,254,254}{\strut  to} \colorbox[RGB]{255,254,254}{\strut  the} \colorbox[RGB]{254,254,255}{\strut  th} \colorbox[RGB]{255,254,254}{\strut ug} \colorbox[RGB]{255,244,244}{\strut .} \colorbox[RGB]{255,254,254}{\strut  And} \colorbox[RGB]{255,250,250}{\strut  I} \colorbox[RGB]{255,244,244}{\strut  also} \colorbox[RGB]{255,248,248}{\strut  like} \colorbox[RGB]{255,252,252}{\strut  the} \colorbox[RGB]{255,250,250}{\strut  part} \colorbox[RGB]{255,250,250}{\strut  where} \colorbox[RGB]{255,254,254}{\strut  the} \colorbox[RGB]{255,254,254}{\strut  bad} \colorbox[RGB]{255,254,254}{\strut  gets} \colorbox[RGB]{254,254,255}{\strut  blow} \colorbox[RGB]{255,252,252}{\strut n} \colorbox[RGB]{255,254,254}{\strut  out} \colorbox[RGB]{255,254,254}{\strut  of} \colorbox[RGB]{255,254,254}{\strut  the} \colorbox[RGB]{255,254,254}{\strut  building} \colorbox[RGB]{255,252,252}{\strut  through} \colorbox[RGB]{255,254,254}{\strut  a} \colorbox[RGB]{255,254,254}{\strut  plate} \colorbox[RGB]{255,254,254}{\strut  glass} \colorbox[RGB]{255,254,254}{\strut  window} \colorbox[RGB]{255,248,248}{\strut .} \colorbox[RGB]{255,252,252}{\strut  That} \colorbox[RGB]{255,250,250}{\strut  was} \colorbox[RGB]{255,252,252}{\strut  the} \colorbox[RGB]{255,252,252}{\strut  B} \colorbox[RGB]{255,252,252}{\strut OM} \colorbox[RGB]{255,250,250}{\strut B} \colorbox[RGB]{255,232,232}{\strut !} \colorbox[RGB]{255,240,240}{\strut  R} \colorbox[RGB]{255,236,236}{\strut andy} \colorbox[RGB]{255,250,250}{\strut  C} \colorbox[RGB]{255,252,252}{\strut raw} \colorbox[RGB]{255,230,230}{\strut ford} \colorbox[RGB]{255,248,248}{\strut '} \colorbox[RGB]{255,244,244}{\strut s} \colorbox[RGB]{255,248,248}{\strut  "} \colorbox[RGB]{255,250,250}{\strut St} \colorbox[RGB]{255,242,242}{\strut reet} \colorbox[RGB]{255,250,250}{\strut  Life} \colorbox[RGB]{255,240,240}{\strut "} \colorbox[RGB]{255,230,230}{\strut  really} \colorbox[RGB]{255,246,246}{\strut  put} \colorbox[RGB]{255,250,250}{\strut  the} \colorbox[RGB]{255,254,254}{\strut  movie} \colorbox[RGB]{255,252,252}{\strut  in} \colorbox[RGB]{255,252,252}{\strut  the} \colorbox[RGB]{255,248,248}{\strut  right} \colorbox[RGB]{255,254,254}{\strut  m} \colorbox[RGB]{255,248,248}{\strut ood} \colorbox[RGB]{255,246,246}{\strut ,} \colorbox[RGB]{255,240,240}{\strut  and} \colorbox[RGB]{255,250,250}{\strut  the} \colorbox[RGB]{255,242,242}{\strut  movie} \colorbox[RGB]{255,242,242}{\strut  itself} \colorbox[RGB]{255,246,246}{\strut  is} \colorbox[RGB]{255,242,242}{\strut  really} \colorbox[RGB]{255,246,246}{\strut  a} \colorbox[RGB]{255,214,214}{\strut  great} \colorbox[RGB]{255,234,234}{\strut  hit} \colorbox[RGB]{255,234,234}{\strut  to} \colorbox[RGB]{255,224,224}{\strut  me} \colorbox[RGB]{255,236,236}{\strut ,} \colorbox[RGB]{255,248,248}{\strut  AL} \colorbox[RGB]{255,248,248}{\strut WA} \colorbox[RGB]{255,216,216}{\strut YS} \colorbox[RGB]{255,184,184}{\strut !} \colorbox[RGB]{255,238,238}{\strut  R} \colorbox[RGB]{255,184,184}{\strut ating} \colorbox[RGB]{255,220,220}{\strut  } \colorbox[RGB]{255,210,210}{\strut 4} \colorbox[RGB]{255,238,238}{\strut  out} \colorbox[RGB]{255,224,224}{\strut  of} \colorbox[RGB]{255,226,226}{\strut  } \colorbox[RGB]{255,147,147}{\strut 5} \colorbox[RGB]{255,156,156}{\strut  stars} \colorbox[RGB]{255,73,73}{\strut .}}}
     

%% file: figures/NLP/samples/3/heatmap_11_LRP+PE+REFORM.tex
    \fbox{
    \parbox{\textwidth}{
    \setlength\fboxsep{0pt}
     \colorbox[RGB]{255,36,36}{\strut <s>} \colorbox[RGB]{255,214,214}{\strut  Great} \colorbox[RGB]{255,244,244}{\strut  just} \colorbox[RGB]{255,226,226}{\strut  great} \colorbox[RGB]{255,216,216}{\strut !} \colorbox[RGB]{255,250,250}{\strut  The} \colorbox[RGB]{227,227,255}{\strut  West} \colorbox[RGB]{255,248,248}{\strut  Coast} \colorbox[RGB]{255,252,252}{\strut  got} \colorbox[RGB]{255,250,250}{\strut  "} \colorbox[RGB]{255,252,252}{\strut Dir} \colorbox[RGB]{255,252,252}{\strut ty} \colorbox[RGB]{255,250,250}{\strut "} \colorbox[RGB]{255,252,252}{\strut  Harry} \colorbox[RGB]{255,250,250}{\strut  Cal} \colorbox[RGB]{255,254,254}{\strut la} \colorbox[RGB]{255,246,246}{\strut han} \colorbox[RGB]{255,248,248}{\strut ,} \colorbox[RGB]{255,252,252}{\strut  the} \colorbox[RGB]{255,254,254}{\strut  East} \colorbox[RGB]{255,252,252}{\strut  Coast} \colorbox[RGB]{255,250,250}{\strut  got} \colorbox[RGB]{255,248,248}{\strut  Sh} \colorbox[RGB]{255,252,252}{\strut ark} \colorbox[RGB]{255,250,250}{\strut y} \colorbox[RGB]{255,244,244}{\strut .} \colorbox[RGB]{255,250,250}{\strut  B} \colorbox[RGB]{255,250,250}{\strut urt} \colorbox[RGB]{252,252,255}{\strut  Reyn} \colorbox[RGB]{255,250,250}{\strut olds} \colorbox[RGB]{255,248,248}{\strut  plays} \colorbox[RGB]{255,250,250}{\strut  Sh} \colorbox[RGB]{255,254,254}{\strut ark} \colorbox[RGB]{255,252,252}{\strut y} \colorbox[RGB]{255,248,248}{\strut  in} \colorbox[RGB]{243,243,255}{\strut  "} \colorbox[RGB]{255,250,250}{\strut Sh} \colorbox[RGB]{255,254,254}{\strut ark} \colorbox[RGB]{255,252,252}{\strut y} \colorbox[RGB]{255,246,246}{\strut '} \colorbox[RGB]{255,252,252}{\strut s} \colorbox[RGB]{255,248,248}{\strut  Machine} \colorbox[RGB]{255,250,250}{\strut "} \colorbox[RGB]{255,246,246}{\strut  and} \colorbox[RGB]{255,242,242}{\strut  I} \colorbox[RGB]{255,236,236}{\strut  enjoyed} \colorbox[RGB]{255,232,232}{\strut  every} \colorbox[RGB]{255,248,248}{\strut  minute} \colorbox[RGB]{255,250,250}{\strut  of} \colorbox[RGB]{255,248,248}{\strut  it} \colorbox[RGB]{255,230,230}{\strut .} \colorbox[RGB]{255,244,244}{\strut  Play} \colorbox[RGB]{255,250,250}{\strut ing} \colorbox[RGB]{255,252,252}{\strut  a} \colorbox[RGB]{254,254,255}{\strut  ma} \colorbox[RGB]{255,254,254}{\strut ver} \colorbox[RGB]{255,252,252}{\strut ick} \colorbox[RGB]{255,254,254}{\strut  n} \colorbox[RGB]{255,254,254}{\strut arc} \colorbox[RGB]{255,252,252}{\strut ot} \colorbox[RGB]{255,254,254}{\strut ics} \colorbox[RGB]{255,250,250}{\strut  cop} \colorbox[RGB]{255,250,250}{\strut  in} \colorbox[RGB]{255,246,246}{\strut  Atlanta} \colorbox[RGB]{255,246,246}{\strut ,} \colorbox[RGB]{255,250,250}{\strut  G} \colorbox[RGB]{255,248,248}{\strut A} \colorbox[RGB]{255,250,250}{\strut  is} \colorbox[RGB]{255,250,250}{\strut  just} \colorbox[RGB]{255,250,250}{\strut  what} \colorbox[RGB]{255,248,248}{\strut  everyone} \colorbox[RGB]{255,250,250}{\strut  wants} \colorbox[RGB]{255,242,242}{\strut .} \colorbox[RGB]{254,254,255}{\strut  Instead} \colorbox[RGB]{255,252,252}{\strut  of} \colorbox[RGB]{255,254,254}{\strut  susp} \colorbox[RGB]{255,254,254}{\strut ension} \colorbox[RGB]{255,250,250}{\strut ,} \colorbox[RGB]{255,250,250}{\strut  he} \colorbox[RGB]{255,252,252}{\strut '} \colorbox[RGB]{255,254,254}{\strut s} \colorbox[RGB]{255,254,254}{\strut  sent} \colorbox[RGB]{255,254,254}{\strut  to} \colorbox[RGB]{255,252,252}{\strut  vice} \colorbox[RGB]{255,254,254}{\strut  squad} \colorbox[RGB]{255,250,250}{\strut .} \colorbox[RGB]{254,254,255}{\strut  Like} \colorbox[RGB]{255,252,252}{\strut  in} \colorbox[RGB]{255,252,252}{\strut  the} \colorbox[RGB]{255,254,254}{\strut  D} \colorbox[RGB]{255,254,254}{\strut irty} \colorbox[RGB]{255,252,252}{\strut  Harry} \colorbox[RGB]{255,254,254}{\strut  mov} \colorbox[RGB]{255,252,252}{\strut ies} \colorbox[RGB]{255,254,254}{\strut  or} \colorbox[RGB]{255,252,252}{\strut  any} \colorbox[RGB]{255,254,254}{\strut  other} \colorbox[RGB]{255,246,246}{\strut  cop} \colorbox[RGB]{255,254,254}{\strut  mov} \colorbox[RGB]{255,254,254}{\strut ies} \colorbox[RGB]{255,250,250}{\strut ,} \colorbox[RGB]{255,254,254}{\strut  the} \colorbox[RGB]{255,252,252}{\strut  captain} \colorbox[RGB]{255,254,254}{\strut  is} \colorbox[RGB]{255,254,254}{\strut  always} \colorbox[RGB]{254,254,255}{\strut  going} \colorbox[RGB]{255,254,254}{\strut  to} \colorbox[RGB]{255,254,254}{\strut  be} \colorbox[RGB]{255,254,254}{\strut  the} \colorbox[RGB]{255,254,254}{\strut  j} \colorbox[RGB]{255,254,254}{\strut erk} \colorbox[RGB]{255,250,250}{\strut .} \colorbox[RGB]{255,252,252}{\strut  When} \colorbox[RGB]{255,250,250}{\strut  I} \colorbox[RGB]{255,254,254}{\strut  was} \colorbox[RGB]{255,254,254}{\strut  a} \colorbox[RGB]{255,252,252}{\strut  kid} \colorbox[RGB]{255,252,252}{\strut ,} \colorbox[RGB]{255,254,254}{\strut  I} \colorbox[RGB]{255,254,254}{\strut  was} \colorbox[RGB]{255,252,252}{\strut  curious} \colorbox[RGB]{255,252,252}{\strut  what} \colorbox[RGB]{255,252,252}{\strut  that} \colorbox[RGB]{255,252,252}{\strut  movie} \colorbox[RGB]{255,254,254}{\strut  meant} \colorbox[RGB]{255,250,250}{\strut  "} \colorbox[RGB]{255,252,252}{\strut Sh} \colorbox[RGB]{255,254,254}{\strut ark} \colorbox[RGB]{255,252,252}{\strut y} \colorbox[RGB]{255,250,250}{\strut '} \colorbox[RGB]{255,254,254}{\strut s} \colorbox[RGB]{255,252,252}{\strut  Machine} \colorbox[RGB]{255,244,244}{\strut ".} \colorbox[RGB]{255,246,246}{\strut  Well} \colorbox[RGB]{255,252,252}{\strut  I} \colorbox[RGB]{255,252,252}{\strut  knew} \colorbox[RGB]{255,254,254}{\strut  who} \colorbox[RGB]{255,254,254}{\strut  played} \colorbox[RGB]{255,254,254}{\strut  Sh} \colorbox[RGB]{255,254,254}{\strut ark} \colorbox[RGB]{255,254,254}{\strut y} \colorbox[RGB]{255,252,252}{\strut ,} \colorbox[RGB]{255,254,254}{\strut  I} \colorbox[RGB]{255,254,254}{\strut  wonder} \colorbox[RGB]{255,254,254}{\strut  what} \colorbox[RGB]{255,254,254}{\strut  his} \colorbox[RGB]{255,254,254}{\strut  machine} \colorbox[RGB]{255,254,254}{\strut  was} \colorbox[RGB]{255,250,250}{\strut .} \colorbox[RGB]{255,252,252}{\strut  It} \colorbox[RGB]{255,254,254}{\strut  was} \colorbox[RGB]{255,252,252}{\strut  his} \colorbox[RGB]{255,250,250}{\strut  GROUP} \colorbox[RGB]{255,252,252}{\strut  of} \colorbox[RGB]{255,252,252}{\strut  fellow} \colorbox[RGB]{255,254,254}{\strut  c} \colorbox[RGB]{255,252,252}{\strut ops} \colorbox[RGB]{255,250,250}{\strut .} \colorbox[RGB]{255,254,254}{\strut  After} \colorbox[RGB]{255,254,254}{\strut  un} \colorbox[RGB]{255,254,254}{\strut cover} \colorbox[RGB]{255,254,254}{\strut ing} \colorbox[RGB]{255,254,254}{\strut  the} \colorbox[RGB]{255,250,250}{\strut  murder} \colorbox[RGB]{255,250,250}{\strut ,} \colorbox[RGB]{255,252,252}{\strut  he} \colorbox[RGB]{255,254,254}{\strut  goes} \colorbox[RGB]{255,252,252}{\strut  all} \colorbox[RGB]{255,252,252}{\strut  out} \colorbox[RGB]{255,252,252}{\strut  to} \colorbox[RGB]{255,254,254}{\strut  find} \colorbox[RGB]{255,254,254}{\strut  the} \colorbox[RGB]{255,254,254}{\strut  per} \colorbox[RGB]{255,254,254}{\strut p} \colorbox[RGB]{255,250,250}{\strut .} \colorbox[RGB]{255,254,254}{\strut  When} \colorbox[RGB]{255,254,254}{\strut  it} \colorbox[RGB]{254,254,255}{\strut  turns} \colorbox[RGB]{255,254,254}{\strut  out} \colorbox[RGB]{255,254,254}{\strut  to} \colorbox[RGB]{255,254,254}{\strut  be} \colorbox[RGB]{255,254,254}{\strut  a} \colorbox[RGB]{255,252,252}{\strut  big} \colorbox[RGB]{255,252,252}{\strut  time} \colorbox[RGB]{255,250,250}{\strut  mob} \colorbox[RGB]{255,254,254}{\strut  b} \colorbox[RGB]{255,254,254}{\strut oss} \colorbox[RGB]{255,252,252}{\strut ,} \colorbox[RGB]{255,252,252}{\strut  Sh} \colorbox[RGB]{255,254,254}{\strut ark} \colorbox[RGB]{255,254,254}{\strut y} \colorbox[RGB]{255,254,254}{\strut  doesn} \colorbox[RGB]{255,250,250}{\strut '} \colorbox[RGB]{255,254,254}{\strut t} \colorbox[RGB]{255,254,254}{\strut  play} \colorbox[RGB]{255,254,254}{\strut  around} \colorbox[RGB]{255,250,250}{\strut .} \colorbox[RGB]{255,252,252}{\strut  When} \colorbox[RGB]{255,254,254}{\strut  he} \colorbox[RGB]{255,254,254}{\strut  gets} \colorbox[RGB]{255,254,254}{\strut  the} \colorbox[RGB]{255,254,254}{\strut  other} \colorbox[RGB]{254,254,255}{\strut  prost} \colorbox[RGB]{255,254,254}{\strut itute} \colorbox[RGB]{255,252,252}{\strut  into} \colorbox[RGB]{255,252,252}{\strut  safety} \colorbox[RGB]{255,252,252}{\strut ,} \colorbox[RGB]{255,254,254}{\strut  Sh} \colorbox[RGB]{255,254,254}{\strut ark} \colorbox[RGB]{255,254,254}{\strut y} \colorbox[RGB]{255,254,254}{\strut  f} \colorbox[RGB]{255,252,252}{\strut ights} \colorbox[RGB]{255,252,252}{\strut  back} \colorbox[RGB]{255,250,250}{\strut  hard} \colorbox[RGB]{255,252,252}{\strut  and} \colorbox[RGB]{255,252,252}{\strut  good} \colorbox[RGB]{254,254,255}{\strut  despite} \colorbox[RGB]{255,254,254}{\strut  losing} \colorbox[RGB]{255,254,254}{\strut  a} \colorbox[RGB]{254,254,255}{\strut  finger} \colorbox[RGB]{255,254,254}{\strut  to} \colorbox[RGB]{255,254,254}{\strut  the} \colorbox[RGB]{254,254,255}{\strut  th} \colorbox[RGB]{255,254,254}{\strut ug} \colorbox[RGB]{255,248,248}{\strut .} \colorbox[RGB]{255,254,254}{\strut  And} \colorbox[RGB]{255,250,250}{\strut  I} \colorbox[RGB]{255,244,244}{\strut  also} \colorbox[RGB]{255,250,250}{\strut  like} \colorbox[RGB]{255,252,252}{\strut  the} \colorbox[RGB]{255,252,252}{\strut  part} \colorbox[RGB]{255,250,250}{\strut  where} \colorbox[RGB]{255,254,254}{\strut  the} \colorbox[RGB]{254,254,255}{\strut  bad} \colorbox[RGB]{255,254,254}{\strut  gets} \colorbox[RGB]{255,254,254}{\strut  blow} \colorbox[RGB]{255,254,254}{\strut n} \colorbox[RGB]{255,254,254}{\strut  out} \colorbox[RGB]{255,254,254}{\strut  of} \colorbox[RGB]{255,254,254}{\strut  the} \colorbox[RGB]{255,254,254}{\strut  building} \colorbox[RGB]{255,252,252}{\strut  through} \colorbox[RGB]{255,254,254}{\strut  a} \colorbox[RGB]{254,254,255}{\strut  plate} \colorbox[RGB]{255,254,254}{\strut  glass} \colorbox[RGB]{255,254,254}{\strut  window} \colorbox[RGB]{255,248,248}{\strut .} \colorbox[RGB]{255,254,254}{\strut  That} \colorbox[RGB]{255,252,252}{\strut  was} \colorbox[RGB]{255,252,252}{\strut  the} \colorbox[RGB]{255,254,254}{\strut  B} \colorbox[RGB]{255,254,254}{\strut OM} \colorbox[RGB]{255,250,250}{\strut B} \colorbox[RGB]{255,230,230}{\strut !} \colorbox[RGB]{255,236,236}{\strut  R} \colorbox[RGB]{255,236,236}{\strut andy} \colorbox[RGB]{255,252,252}{\strut  C} \colorbox[RGB]{255,252,252}{\strut raw} \colorbox[RGB]{255,238,238}{\strut ford} \colorbox[RGB]{255,248,248}{\strut '} \colorbox[RGB]{255,244,244}{\strut s} \colorbox[RGB]{255,242,242}{\strut  "} \colorbox[RGB]{255,250,250}{\strut St} \colorbox[RGB]{255,240,240}{\strut reet} \colorbox[RGB]{255,250,250}{\strut  Life} \colorbox[RGB]{255,240,240}{\strut "} \colorbox[RGB]{255,232,232}{\strut  really} \colorbox[RGB]{255,250,250}{\strut  put} \colorbox[RGB]{255,250,250}{\strut  the} \colorbox[RGB]{255,248,248}{\strut  movie} \colorbox[RGB]{255,252,252}{\strut  in} \colorbox[RGB]{255,254,254}{\strut  the} \colorbox[RGB]{255,248,248}{\strut  right} \colorbox[RGB]{255,254,254}{\strut  m} \colorbox[RGB]{255,250,250}{\strut ood} \colorbox[RGB]{255,244,244}{\strut ,} \colorbox[RGB]{255,240,240}{\strut  and} \colorbox[RGB]{255,250,250}{\strut  the} \colorbox[RGB]{255,244,244}{\strut  movie} \colorbox[RGB]{255,242,242}{\strut  itself} \colorbox[RGB]{255,248,248}{\strut  is} \colorbox[RGB]{255,238,238}{\strut  really} \colorbox[RGB]{255,244,244}{\strut  a} \colorbox[RGB]{255,208,208}{\strut  great} \colorbox[RGB]{255,230,230}{\strut  hit} \colorbox[RGB]{255,236,236}{\strut  to} \colorbox[RGB]{255,224,224}{\strut  me} \colorbox[RGB]{255,238,238}{\strut ,} \colorbox[RGB]{255,244,244}{\strut  AL} \colorbox[RGB]{255,242,242}{\strut WA} \colorbox[RGB]{255,220,220}{\strut YS} \colorbox[RGB]{255,174,174}{\strut !} \colorbox[RGB]{255,238,238}{\strut  R} \colorbox[RGB]{255,188,188}{\strut ating} \colorbox[RGB]{255,226,226}{\strut  } \colorbox[RGB]{255,202,202}{\strut 4} \colorbox[RGB]{255,238,238}{\strut  out} \colorbox[RGB]{255,226,226}{\strut  of} \colorbox[RGB]{255,220,220}{\strut  } \colorbox[RGB]{255,138,138}{\strut 5} \colorbox[RGB]{255,158,158}{\strut  stars} \colorbox[RGB]{255,0,0}{\strut .}}}
     

%% file: figures/NLP/samples/7/heatmap_25_baseline.tex
    \fbox{
    \parbox{\textwidth}{
    \setlength\fboxsep{0pt}
     \colorbox[RGB]{255,0,0}{\strut <s>} \colorbox[RGB]{255,250,250}{\strut  I} \colorbox[RGB]{252,252,255}{\strut  went} \colorbox[RGB]{254,254,255}{\strut  to} \colorbox[RGB]{252,252,255}{\strut  see} \colorbox[RGB]{255,246,246}{\strut  Ham} \colorbox[RGB]{255,224,224}{\strut let} \colorbox[RGB]{248,248,255}{\strut  because} \colorbox[RGB]{255,252,252}{\strut  I} \colorbox[RGB]{255,252,252}{\strut  was} \colorbox[RGB]{255,254,254}{\strut  in} \colorbox[RGB]{255,252,252}{\strut  between} \colorbox[RGB]{243,243,255}{\strut  jobs} \colorbox[RGB]{255,254,254}{\strut .} \colorbox[RGB]{255,252,252}{\strut  I} \colorbox[RGB]{246,246,255}{\strut  figured} \colorbox[RGB]{255,254,254}{\strut  } \colorbox[RGB]{248,248,255}{\strut 4} \colorbox[RGB]{252,252,255}{\strut  hours} \colorbox[RGB]{255,250,250}{\strut  would} \colorbox[RGB]{255,254,254}{\strut  be} \colorbox[RGB]{255,238,238}{\strut  great} \colorbox[RGB]{255,248,248}{\strut ,} \colorbox[RGB]{255,252,252}{\strut  I} \colorbox[RGB]{255,252,252}{\strut '} \colorbox[RGB]{255,254,254}{\strut ve} \colorbox[RGB]{255,248,248}{\strut  been} \colorbox[RGB]{255,252,252}{\strut  a} \colorbox[RGB]{255,242,242}{\strut  fan} \colorbox[RGB]{255,248,248}{\strut  of} \colorbox[RGB]{230,230,255}{\strut  Bran} \colorbox[RGB]{255,226,226}{\strut agh} \colorbox[RGB]{255,232,232}{\strut ;} \colorbox[RGB]{255,220,220}{\strut  Dead} \colorbox[RGB]{255,252,252}{\strut  Again} \colorbox[RGB]{255,250,250}{\strut ,} \colorbox[RGB]{255,228,228}{\strut  Henry} \colorbox[RGB]{255,254,254}{\strut  V} \colorbox[RGB]{255,250,250}{\strut .} \colorbox[RGB]{255,244,244}{\strut  I} \colorbox[RGB]{255,248,248}{\strut  was} \colorbox[RGB]{255,254,254}{\strut  completely} \colorbox[RGB]{246,246,255}{\strut  over} \colorbox[RGB]{255,250,250}{\strut wh} \colorbox[RGB]{255,254,254}{\strut el} \colorbox[RGB]{254,254,255}{\strut med} \colorbox[RGB]{255,244,244}{\strut  by} \colorbox[RGB]{255,248,248}{\strut  the} \colorbox[RGB]{248,248,255}{\strut  direction} \colorbox[RGB]{255,242,242}{\strut ,} \colorbox[RGB]{255,252,252}{\strut  acting} \colorbox[RGB]{255,248,248}{\strut ,} \colorbox[RGB]{255,244,244}{\strut  cinemat} \colorbox[RGB]{255,254,254}{\strut ography} \colorbox[RGB]{255,246,246}{\strut  that} \colorbox[RGB]{255,242,242}{\strut  this} \colorbox[RGB]{255,246,246}{\strut  film} \colorbox[RGB]{255,224,224}{\strut  captured} \colorbox[RGB]{255,220,220}{\strut .} \colorbox[RGB]{216,216,255}{\strut  Like} \colorbox[RGB]{255,250,250}{\strut  other} \colorbox[RGB]{232,232,255}{\strut  reviews} \colorbox[RGB]{255,246,246}{\strut  the} \colorbox[RGB]{254,254,255}{\strut  } \colorbox[RGB]{252,252,255}{\strut 4} \colorbox[RGB]{248,248,255}{\strut  hours} \colorbox[RGB]{255,252,252}{\strut  passes} \colorbox[RGB]{255,236,236}{\strut  swift} \colorbox[RGB]{255,246,246}{\strut ly} \colorbox[RGB]{255,246,246}{\strut .} \colorbox[RGB]{255,195,195}{\strut  Bran} \colorbox[RGB]{255,224,224}{\strut agh} \colorbox[RGB]{252,252,255}{\strut  doesn} \colorbox[RGB]{255,242,242}{\strut '} \colorbox[RGB]{250,250,255}{\strut t} \colorbox[RGB]{252,252,255}{\strut  play} \colorbox[RGB]{240,240,255}{\strut  Ham} \colorbox[RGB]{255,252,252}{\strut let} \colorbox[RGB]{255,254,254}{\strut ,} \colorbox[RGB]{255,248,248}{\strut  he} \colorbox[RGB]{255,250,250}{\strut  is} \colorbox[RGB]{255,254,254}{\strut  Ham} \colorbox[RGB]{255,252,252}{\strut let} \colorbox[RGB]{255,248,248}{\strut ,} \colorbox[RGB]{255,250,250}{\strut  he} \colorbox[RGB]{255,254,254}{\strut  was} \colorbox[RGB]{255,240,240}{\strut  born} \colorbox[RGB]{255,236,236}{\strut  for} \colorbox[RGB]{255,240,240}{\strut  this} \colorbox[RGB]{255,220,220}{\strut .} \colorbox[RGB]{255,244,244}{\strut  When} \colorbox[RGB]{255,248,248}{\strut  I} \colorbox[RGB]{255,252,252}{\strut  watch} \colorbox[RGB]{255,240,240}{\strut  this} \colorbox[RGB]{255,246,246}{\strut  film} \colorbox[RGB]{254,254,255}{\strut  I} \colorbox[RGB]{234,234,255}{\strut '} \colorbox[RGB]{252,252,255}{\strut m} \colorbox[RGB]{254,254,255}{\strut  constantly} \colorbox[RGB]{254,254,255}{\strut  trying} \colorbox[RGB]{255,254,254}{\strut  to} \colorbox[RGB]{255,252,252}{\strut  find} \colorbox[RGB]{255,250,250}{\strut  fault} \colorbox[RGB]{255,252,252}{\strut s} \colorbox[RGB]{255,254,254}{\strut ,} \colorbox[RGB]{254,254,255}{\strut  I} \colorbox[RGB]{248,248,255}{\strut '} \colorbox[RGB]{254,254,255}{\strut ve} \colorbox[RGB]{254,254,255}{\strut  looked} \colorbox[RGB]{254,254,255}{\strut  at} \colorbox[RGB]{255,254,254}{\strut  the} \colorbox[RGB]{254,254,255}{\strut  go} \colorbox[RGB]{252,252,255}{\strut of} \colorbox[RGB]{255,254,254}{\strut s} \colorbox[RGB]{255,252,252}{\strut  and} \colorbox[RGB]{252,252,255}{\strut  haven} \colorbox[RGB]{255,250,250}{\strut '} \colorbox[RGB]{252,252,255}{\strut t} \colorbox[RGB]{252,252,255}{\strut  noticed} \colorbox[RGB]{255,254,254}{\strut  them} \colorbox[RGB]{255,246,246}{\strut .} \colorbox[RGB]{227,227,255}{\strut  How} \colorbox[RGB]{255,246,246}{\strut  he} \colorbox[RGB]{254,254,255}{\strut  was} \colorbox[RGB]{255,252,252}{\strut  able} \colorbox[RGB]{254,254,255}{\strut  to} \colorbox[RGB]{252,252,255}{\strut  move} \colorbox[RGB]{255,254,254}{\strut  the} \colorbox[RGB]{254,254,255}{\strut  camera} \colorbox[RGB]{255,254,254}{\strut  in} \colorbox[RGB]{255,252,252}{\strut  and} \colorbox[RGB]{255,254,254}{\strut  out} \colorbox[RGB]{255,254,254}{\strut  of} \colorbox[RGB]{255,252,252}{\strut  the} \colorbox[RGB]{255,250,250}{\strut  Hall} \colorbox[RGB]{255,250,250}{\strut  with} \colorbox[RGB]{255,254,254}{\strut  all} \colorbox[RGB]{255,252,252}{\strut  the} \colorbox[RGB]{255,248,248}{\strut  mirror} \colorbox[RGB]{255,254,254}{\strut s} \colorbox[RGB]{248,248,255}{\strut  is} \colorbox[RGB]{252,252,255}{\strut  a} \colorbox[RGB]{238,238,255}{\strut  mystery} \colorbox[RGB]{252,252,255}{\strut  to} \colorbox[RGB]{248,248,255}{\strut  me} \colorbox[RGB]{246,246,255}{\strut .} \colorbox[RGB]{246,246,255}{\strut  This} \colorbox[RGB]{240,240,255}{\strut  movie} \colorbox[RGB]{250,250,255}{\strut  was} \colorbox[RGB]{255,246,246}{\strut  shot} \colorbox[RGB]{255,254,254}{\strut  in} \colorbox[RGB]{255,254,254}{\strut  } \colorbox[RGB]{255,252,252}{\strut 7} \colorbox[RGB]{255,254,254}{\strut 0} \colorbox[RGB]{255,248,248}{\strut  mil} \colorbox[RGB]{255,244,244}{\strut .} \colorbox[RGB]{255,254,254}{\strut  It} \colorbox[RGB]{255,230,230}{\strut '} \colorbox[RGB]{252,252,255}{\strut s} \colorbox[RGB]{255,254,254}{\strut  a} \colorbox[RGB]{232,232,255}{\strut  shame} \colorbox[RGB]{255,252,252}{\strut  that} \colorbox[RGB]{255,238,238}{\strut  Columbia} \colorbox[RGB]{248,248,255}{\strut  hasn} \colorbox[RGB]{232,232,255}{\strut '} \colorbox[RGB]{252,252,255}{\strut t} \colorbox[RGB]{252,252,255}{\strut  released} \colorbox[RGB]{255,254,254}{\strut  a} \colorbox[RGB]{255,252,252}{\strut  W} \colorbox[RGB]{255,250,250}{\strut ides} \colorbox[RGB]{255,240,240}{\strut creen} \colorbox[RGB]{255,252,252}{\strut  version} \colorbox[RGB]{255,252,252}{\strut  of} \colorbox[RGB]{255,246,246}{\strut  this} \colorbox[RGB]{255,248,248}{\strut  on} \colorbox[RGB]{255,238,238}{\strut  V} \colorbox[RGB]{255,230,230}{\strut HS} \colorbox[RGB]{255,246,246}{\strut .} \colorbox[RGB]{255,254,254}{\strut  I} \colorbox[RGB]{255,254,254}{\strut  own} \colorbox[RGB]{254,254,255}{\strut  a} \colorbox[RGB]{255,254,254}{\strut  DVD} \colorbox[RGB]{252,252,255}{\strut  player} \colorbox[RGB]{255,250,250}{\strut ,} \colorbox[RGB]{255,250,250}{\strut  and} \colorbox[RGB]{255,244,244}{\strut  I} \colorbox[RGB]{255,210,210}{\strut '} \colorbox[RGB]{254,254,255}{\strut d} \colorbox[RGB]{254,254,255}{\strut  take} \colorbox[RGB]{255,246,246}{\strut  this} \colorbox[RGB]{252,252,255}{\strut  over} \colorbox[RGB]{254,254,255}{\strut  T} \colorbox[RGB]{254,254,255}{\strut itan} \colorbox[RGB]{254,254,255}{\strut ic} \colorbox[RGB]{255,254,254}{\strut  any} \colorbox[RGB]{255,254,254}{\strut  day} \colorbox[RGB]{255,250,250}{\strut .} \colorbox[RGB]{255,250,250}{\strut  So} \colorbox[RGB]{255,236,236}{\strut  Columbia} \colorbox[RGB]{255,244,244}{\strut  if} \colorbox[RGB]{255,244,244}{\strut  you} \colorbox[RGB]{255,240,240}{\strut '} \colorbox[RGB]{254,254,255}{\strut re} \colorbox[RGB]{255,246,246}{\strut  listening} \colorbox[RGB]{238,238,255}{\strut  put} \colorbox[RGB]{255,246,246}{\strut  this} \colorbox[RGB]{255,244,244}{\strut  film} \colorbox[RGB]{252,252,255}{\strut  out} \colorbox[RGB]{255,254,254}{\strut  the} \colorbox[RGB]{252,252,255}{\strut  way} \colorbox[RGB]{255,254,254}{\strut  it} \colorbox[RGB]{252,252,255}{\strut  should} \colorbox[RGB]{254,254,255}{\strut  be} \colorbox[RGB]{255,252,252}{\strut  watched} \colorbox[RGB]{255,97,97}{\strut !} \colorbox[RGB]{211,211,255}{\strut  And} \colorbox[RGB]{255,248,248}{\strut  I} \colorbox[RGB]{255,250,250}{\strut  don} \colorbox[RGB]{255,234,234}{\strut '} \colorbox[RGB]{255,254,254}{\strut t} \colorbox[RGB]{255,244,244}{\strut  know} \colorbox[RGB]{254,254,255}{\strut  what} \colorbox[RGB]{254,254,255}{\strut  happened} \colorbox[RGB]{255,254,254}{\strut  at} \colorbox[RGB]{255,252,252}{\strut  the} \colorbox[RGB]{252,252,255}{\strut  O} \colorbox[RGB]{252,252,255}{\strut sc} \colorbox[RGB]{252,252,255}{\strut ars} \colorbox[RGB]{254,254,255}{\strut .} \colorbox[RGB]{255,254,254}{\strut  This} \colorbox[RGB]{220,220,255}{\strut  should} \colorbox[RGB]{234,234,255}{\strut  have} \colorbox[RGB]{255,248,248}{\strut  swe} \colorbox[RGB]{255,252,252}{\strut pt} \colorbox[RGB]{255,252,252}{\strut  Best} \colorbox[RGB]{255,252,252}{\strut  Picture} \colorbox[RGB]{255,246,246}{\strut ,} \colorbox[RGB]{254,254,255}{\strut  Best} \colorbox[RGB]{255,254,254}{\strut  A} \colorbox[RGB]{255,254,254}{\strut ctor} \colorbox[RGB]{255,252,252}{\strut ,} \colorbox[RGB]{255,254,254}{\strut  Best} \colorbox[RGB]{255,254,254}{\strut  D} \colorbox[RGB]{254,254,255}{\strut irection} \colorbox[RGB]{255,252,252}{\strut ,} \colorbox[RGB]{254,254,255}{\strut  best} \colorbox[RGB]{255,252,252}{\strut  cinemat} \colorbox[RGB]{255,254,254}{\strut ography} \colorbox[RGB]{238,238,255}{\strut .} \colorbox[RGB]{224,224,255}{\strut  What} \colorbox[RGB]{255,254,254}{\strut  films} \colorbox[RGB]{255,250,250}{\strut  were} \colorbox[RGB]{254,254,255}{\strut  they} \colorbox[RGB]{250,250,255}{\strut  watching} \colorbox[RGB]{255,242,242}{\strut ?} \colorbox[RGB]{255,238,238}{\strut  I} \colorbox[RGB]{255,234,234}{\strut  felt} \colorbox[RGB]{238,238,255}{\strut  sorry} \colorbox[RGB]{250,250,255}{\strut  for} \colorbox[RGB]{255,226,226}{\strut  Bran} \colorbox[RGB]{255,248,248}{\strut agh} \colorbox[RGB]{255,244,244}{\strut  at} \colorbox[RGB]{255,252,252}{\strut  the} \colorbox[RGB]{255,252,252}{\strut  O} \colorbox[RGB]{255,250,250}{\strut sc} \colorbox[RGB]{255,246,246}{\strut ars} \colorbox[RGB]{255,216,216}{\strut  when} \colorbox[RGB]{255,250,250}{\strut  he} \colorbox[RGB]{255,252,252}{\strut  did} \colorbox[RGB]{255,250,250}{\strut  a} \colorbox[RGB]{254,254,255}{\strut  t} \colorbox[RGB]{254,254,255}{\strut ribute} \colorbox[RGB]{255,246,246}{\strut  to} \colorbox[RGB]{255,142,142}{\strut  Shakespeare} \colorbox[RGB]{255,250,250}{\strut  on} \colorbox[RGB]{255,254,254}{\strut  the} \colorbox[RGB]{254,254,255}{\strut  screen} \colorbox[RGB]{255,240,240}{\strut .} \colorbox[RGB]{226,226,255}{\strut  They} \colorbox[RGB]{124,124,255}{\strut  should} \colorbox[RGB]{176,176,255}{\strut  have} \colorbox[RGB]{198,198,255}{\strut  been} \colorbox[RGB]{255,252,252}{\strut  giving} \colorbox[RGB]{255,250,250}{\strut  a} \colorbox[RGB]{255,252,252}{\strut  t} \colorbox[RGB]{255,234,234}{\strut ribute} \colorbox[RGB]{255,244,244}{\strut  to} \colorbox[RGB]{255,208,208}{\strut  Bran} \colorbox[RGB]{255,252,252}{\strut agh} \colorbox[RGB]{255,210,210}{\strut  for} \colorbox[RGB]{255,179,179}{\strut  bringing} \colorbox[RGB]{255,206,206}{\strut  us} \colorbox[RGB]{255,238,238}{\strut  one} \colorbox[RGB]{255,242,242}{\strut  of} \colorbox[RGB]{255,248,248}{\strut  the} \colorbox[RGB]{255,244,244}{\strut  greatest} \colorbox[RGB]{255,244,244}{\strut  films} \colorbox[RGB]{255,254,254}{\strut  of} \colorbox[RGB]{254,254,255}{\strut  all} \colorbox[RGB]{255,244,244}{\strut  time} \colorbox[RGB]{255,44,44}{\strut .}}}
     

%% file: figures/NLP/samples/7/heatmap_25_PE.tex
    \fbox{
    \parbox{\textwidth}{
    \setlength\fboxsep{0pt}
     \colorbox[RGB]{255,17,17}{\strut <s>} \colorbox[RGB]{255,248,248}{\strut  I} \colorbox[RGB]{254,254,255}{\strut  went} \colorbox[RGB]{255,252,252}{\strut  to} \colorbox[RGB]{255,252,252}{\strut  see} \colorbox[RGB]{255,224,224}{\strut  Ham} \colorbox[RGB]{255,222,222}{\strut let} \colorbox[RGB]{255,252,252}{\strut  because} \colorbox[RGB]{255,252,252}{\strut  I} \colorbox[RGB]{255,252,252}{\strut  was} \colorbox[RGB]{255,252,252}{\strut  in} \colorbox[RGB]{255,252,252}{\strut  between} \colorbox[RGB]{252,252,255}{\strut  jobs} \colorbox[RGB]{255,246,246}{\strut .} \colorbox[RGB]{255,250,250}{\strut  I} \colorbox[RGB]{254,254,255}{\strut  figured} \colorbox[RGB]{255,254,254}{\strut  } \colorbox[RGB]{254,254,255}{\strut 4} \colorbox[RGB]{255,254,254}{\strut  hours} \colorbox[RGB]{255,250,250}{\strut  would} \colorbox[RGB]{255,254,254}{\strut  be} \colorbox[RGB]{255,240,240}{\strut  great} \colorbox[RGB]{255,248,248}{\strut ,} \colorbox[RGB]{255,250,250}{\strut  I} \colorbox[RGB]{255,248,248}{\strut '} \colorbox[RGB]{255,252,252}{\strut ve} \colorbox[RGB]{255,250,250}{\strut  been} \colorbox[RGB]{255,252,252}{\strut  a} \colorbox[RGB]{255,242,242}{\strut  fan} \colorbox[RGB]{255,246,246}{\strut  of} \colorbox[RGB]{255,244,244}{\strut  Bran} \colorbox[RGB]{255,228,228}{\strut agh} \colorbox[RGB]{255,240,240}{\strut ;} \colorbox[RGB]{255,234,234}{\strut  Dead} \colorbox[RGB]{255,250,250}{\strut  Again} \colorbox[RGB]{255,250,250}{\strut ,} \colorbox[RGB]{255,236,236}{\strut  Henry} \colorbox[RGB]{255,252,252}{\strut  V} \colorbox[RGB]{255,246,246}{\strut .} \colorbox[RGB]{255,246,246}{\strut  I} \colorbox[RGB]{255,248,248}{\strut  was} \colorbox[RGB]{255,248,248}{\strut  completely} \colorbox[RGB]{255,254,254}{\strut  over} \colorbox[RGB]{255,248,248}{\strut wh} \colorbox[RGB]{255,254,254}{\strut el} \colorbox[RGB]{255,252,252}{\strut med} \colorbox[RGB]{255,244,244}{\strut  by} \colorbox[RGB]{255,248,248}{\strut  the} \colorbox[RGB]{255,254,254}{\strut  direction} \colorbox[RGB]{255,244,244}{\strut ,} \colorbox[RGB]{255,252,252}{\strut  acting} \colorbox[RGB]{255,248,248}{\strut ,} \colorbox[RGB]{255,248,248}{\strut  cinemat} \colorbox[RGB]{255,254,254}{\strut ography} \colorbox[RGB]{255,248,248}{\strut  that} \colorbox[RGB]{255,244,244}{\strut  this} \colorbox[RGB]{255,246,246}{\strut  film} \colorbox[RGB]{255,234,234}{\strut  captured} \colorbox[RGB]{255,222,222}{\strut .} \colorbox[RGB]{240,240,255}{\strut  Like} \colorbox[RGB]{255,246,246}{\strut  other} \colorbox[RGB]{252,252,255}{\strut  reviews} \colorbox[RGB]{255,240,240}{\strut  the} \colorbox[RGB]{255,254,254}{\strut  } \colorbox[RGB]{255,252,252}{\strut 4} \colorbox[RGB]{255,252,252}{\strut  hours} \colorbox[RGB]{255,250,250}{\strut  passes} \colorbox[RGB]{255,242,242}{\strut  swift} \colorbox[RGB]{255,246,246}{\strut ly} \colorbox[RGB]{255,248,248}{\strut .} \colorbox[RGB]{255,222,222}{\strut  Bran} \colorbox[RGB]{255,232,232}{\strut agh} \colorbox[RGB]{254,254,255}{\strut  doesn} \colorbox[RGB]{255,242,242}{\strut '} \colorbox[RGB]{255,254,254}{\strut t} \colorbox[RGB]{255,254,254}{\strut  play} \colorbox[RGB]{255,254,254}{\strut  Ham} \colorbox[RGB]{255,250,250}{\strut let} \colorbox[RGB]{255,250,250}{\strut ,} \colorbox[RGB]{255,250,250}{\strut  he} \colorbox[RGB]{255,252,252}{\strut  is} \colorbox[RGB]{255,252,252}{\strut  Ham} \colorbox[RGB]{255,250,250}{\strut let} \colorbox[RGB]{255,250,250}{\strut ,} \colorbox[RGB]{255,252,252}{\strut  he} \colorbox[RGB]{255,254,254}{\strut  was} \colorbox[RGB]{255,244,244}{\strut  born} \colorbox[RGB]{255,244,244}{\strut  for} \colorbox[RGB]{255,246,246}{\strut  this} \colorbox[RGB]{255,228,228}{\strut .} \colorbox[RGB]{255,246,246}{\strut  When} \colorbox[RGB]{255,248,248}{\strut  I} \colorbox[RGB]{255,252,252}{\strut  watch} \colorbox[RGB]{255,244,244}{\strut  this} \colorbox[RGB]{255,246,246}{\strut  film} \colorbox[RGB]{255,248,248}{\strut  I} \colorbox[RGB]{254,254,255}{\strut '} \colorbox[RGB]{254,254,255}{\strut m} \colorbox[RGB]{255,252,252}{\strut  constantly} \colorbox[RGB]{255,254,254}{\strut  trying} \colorbox[RGB]{255,252,252}{\strut  to} \colorbox[RGB]{255,252,252}{\strut  find} \colorbox[RGB]{255,248,248}{\strut  fault} \colorbox[RGB]{255,252,252}{\strut s} \colorbox[RGB]{255,250,250}{\strut ,} \colorbox[RGB]{255,254,254}{\strut  I} \colorbox[RGB]{255,252,252}{\strut '} \colorbox[RGB]{255,254,254}{\strut ve} \colorbox[RGB]{255,254,254}{\strut  looked} \colorbox[RGB]{255,254,254}{\strut  at} \colorbox[RGB]{255,252,252}{\strut  the} \colorbox[RGB]{255,254,254}{\strut  go} \colorbox[RGB]{255,252,252}{\strut of} \colorbox[RGB]{255,252,252}{\strut s} \colorbox[RGB]{255,250,250}{\strut  and} \colorbox[RGB]{255,254,254}{\strut  haven} \colorbox[RGB]{255,250,250}{\strut '} \colorbox[RGB]{254,254,255}{\strut t} \colorbox[RGB]{255,252,252}{\strut  noticed} \colorbox[RGB]{255,254,254}{\strut  them} \colorbox[RGB]{255,240,240}{\strut .} \colorbox[RGB]{248,248,255}{\strut  How} \colorbox[RGB]{255,246,246}{\strut  he} \colorbox[RGB]{255,254,254}{\strut  was} \colorbox[RGB]{255,248,248}{\strut  able} \colorbox[RGB]{255,254,254}{\strut  to} \colorbox[RGB]{254,254,255}{\strut  move} \colorbox[RGB]{255,254,254}{\strut  the} \colorbox[RGB]{255,254,254}{\strut  camera} \colorbox[RGB]{255,254,254}{\strut  in} \colorbox[RGB]{255,252,252}{\strut  and} \colorbox[RGB]{255,254,254}{\strut  out} \colorbox[RGB]{255,254,254}{\strut  of} \colorbox[RGB]{255,252,252}{\strut  the} \colorbox[RGB]{255,250,250}{\strut  Hall} \colorbox[RGB]{255,250,250}{\strut  with} \colorbox[RGB]{255,254,254}{\strut  all} \colorbox[RGB]{255,252,252}{\strut  the} \colorbox[RGB]{255,252,252}{\strut  mirror} \colorbox[RGB]{255,254,254}{\strut s} \colorbox[RGB]{255,252,252}{\strut  is} \colorbox[RGB]{255,252,252}{\strut  a} \colorbox[RGB]{252,252,255}{\strut  mystery} \colorbox[RGB]{255,254,254}{\strut  to} \colorbox[RGB]{255,250,250}{\strut  me} \colorbox[RGB]{255,248,248}{\strut .} \colorbox[RGB]{252,252,255}{\strut  This} \colorbox[RGB]{254,254,255}{\strut  movie} \colorbox[RGB]{255,254,254}{\strut  was} \colorbox[RGB]{255,248,248}{\strut  shot} \colorbox[RGB]{255,252,252}{\strut  in} \colorbox[RGB]{255,254,254}{\strut  } \colorbox[RGB]{255,252,252}{\strut 7} \colorbox[RGB]{255,252,252}{\strut 0} \colorbox[RGB]{255,246,246}{\strut  mil} \colorbox[RGB]{255,240,240}{\strut .} \colorbox[RGB]{255,250,250}{\strut  It} \colorbox[RGB]{255,232,232}{\strut '} \colorbox[RGB]{255,254,254}{\strut s} \colorbox[RGB]{255,254,254}{\strut  a} \colorbox[RGB]{255,254,254}{\strut  shame} \colorbox[RGB]{255,250,250}{\strut  that} \colorbox[RGB]{255,242,242}{\strut  Columbia} \colorbox[RGB]{254,254,255}{\strut  hasn} \colorbox[RGB]{255,254,254}{\strut '} \colorbox[RGB]{255,254,254}{\strut t} \colorbox[RGB]{255,252,252}{\strut  released} \colorbox[RGB]{255,254,254}{\strut  a} \colorbox[RGB]{255,252,252}{\strut  W} \colorbox[RGB]{255,250,250}{\strut ides} \colorbox[RGB]{255,240,240}{\strut creen} \colorbox[RGB]{255,250,250}{\strut  version} \colorbox[RGB]{255,252,252}{\strut  of} \colorbox[RGB]{255,248,248}{\strut  this} \colorbox[RGB]{255,248,248}{\strut  on} \colorbox[RGB]{255,242,242}{\strut  V} \colorbox[RGB]{255,232,232}{\strut HS} \colorbox[RGB]{255,242,242}{\strut .} \colorbox[RGB]{255,248,248}{\strut  I} \colorbox[RGB]{255,250,250}{\strut  own} \colorbox[RGB]{255,252,252}{\strut  a} \colorbox[RGB]{255,248,248}{\strut  DVD} \colorbox[RGB]{255,250,250}{\strut  player} \colorbox[RGB]{255,248,248}{\strut ,} \colorbox[RGB]{255,246,246}{\strut  and} \colorbox[RGB]{255,246,246}{\strut  I} \colorbox[RGB]{255,218,218}{\strut '} \colorbox[RGB]{255,254,254}{\strut d} \colorbox[RGB]{255,254,254}{\strut  take} \colorbox[RGB]{255,248,248}{\strut  this} \colorbox[RGB]{255,250,250}{\strut  over} \colorbox[RGB]{255,254,254}{\strut  T} \colorbox[RGB]{255,254,254}{\strut itan} \colorbox[RGB]{255,254,254}{\strut ic} \colorbox[RGB]{255,252,252}{\strut  any} \colorbox[RGB]{255,252,252}{\strut  day} \colorbox[RGB]{255,242,242}{\strut .} \colorbox[RGB]{255,246,246}{\strut  So} \colorbox[RGB]{255,238,238}{\strut  Columbia} \colorbox[RGB]{255,242,242}{\strut  if} \colorbox[RGB]{255,246,246}{\strut  you} \colorbox[RGB]{255,240,240}{\strut '} \colorbox[RGB]{255,254,254}{\strut re} \colorbox[RGB]{255,244,244}{\strut  listening} \colorbox[RGB]{254,254,255}{\strut  put} \colorbox[RGB]{255,244,244}{\strut  this} \colorbox[RGB]{255,246,246}{\strut  film} \colorbox[RGB]{255,252,252}{\strut  out} \colorbox[RGB]{255,252,252}{\strut  the} \colorbox[RGB]{255,246,246}{\strut  way} \colorbox[RGB]{255,252,252}{\strut  it} \colorbox[RGB]{254,254,255}{\strut  should} \colorbox[RGB]{255,250,250}{\strut  be} \colorbox[RGB]{255,248,248}{\strut  watched} \colorbox[RGB]{255,152,152}{\strut !} \colorbox[RGB]{250,250,255}{\strut  And} \colorbox[RGB]{255,246,246}{\strut  I} \colorbox[RGB]{255,250,250}{\strut  don} \colorbox[RGB]{255,234,234}{\strut '} \colorbox[RGB]{255,250,250}{\strut t} \colorbox[RGB]{255,246,246}{\strut  know} \colorbox[RGB]{255,252,252}{\strut  what} \colorbox[RGB]{255,250,250}{\strut  happened} \colorbox[RGB]{255,252,252}{\strut  at} \colorbox[RGB]{255,250,250}{\strut  the} \colorbox[RGB]{255,252,252}{\strut  O} \colorbox[RGB]{255,254,254}{\strut sc} \colorbox[RGB]{255,248,248}{\strut ars} \colorbox[RGB]{255,240,240}{\strut .} \colorbox[RGB]{255,248,248}{\strut  This} \colorbox[RGB]{246,246,255}{\strut  should} \colorbox[RGB]{255,252,252}{\strut  have} \colorbox[RGB]{255,248,248}{\strut  swe} \colorbox[RGB]{255,246,246}{\strut pt} \colorbox[RGB]{255,250,250}{\strut  Best} \colorbox[RGB]{255,250,250}{\strut  Picture} \colorbox[RGB]{255,242,242}{\strut ,} \colorbox[RGB]{255,254,254}{\strut  Best} \colorbox[RGB]{255,254,254}{\strut  A} \colorbox[RGB]{255,252,252}{\strut ctor} \colorbox[RGB]{255,250,250}{\strut ,} \colorbox[RGB]{255,254,254}{\strut  Best} \colorbox[RGB]{255,254,254}{\strut  D} \colorbox[RGB]{255,254,254}{\strut irection} \colorbox[RGB]{255,250,250}{\strut ,} \colorbox[RGB]{255,250,250}{\strut  best} \colorbox[RGB]{255,252,252}{\strut  cinemat} \colorbox[RGB]{255,252,252}{\strut ography} \colorbox[RGB]{255,238,238}{\strut .} \colorbox[RGB]{248,248,255}{\strut  What} \colorbox[RGB]{255,242,242}{\strut  films} \colorbox[RGB]{255,244,244}{\strut  were} \colorbox[RGB]{255,246,246}{\strut  they} \colorbox[RGB]{255,244,244}{\strut  watching} \colorbox[RGB]{255,230,230}{\strut ?} \colorbox[RGB]{255,232,232}{\strut  I} \colorbox[RGB]{255,234,234}{\strut  felt} \colorbox[RGB]{255,254,254}{\strut  sorry} \colorbox[RGB]{255,242,242}{\strut  for} \colorbox[RGB]{255,230,230}{\strut  Bran} \colorbox[RGB]{255,240,240}{\strut agh} \colorbox[RGB]{255,240,240}{\strut  at} \colorbox[RGB]{255,250,250}{\strut  the} \colorbox[RGB]{255,246,246}{\strut  O} \colorbox[RGB]{255,248,248}{\strut sc} \colorbox[RGB]{255,240,240}{\strut ars} \colorbox[RGB]{255,220,220}{\strut  when} \colorbox[RGB]{255,244,244}{\strut  he} \colorbox[RGB]{255,248,248}{\strut  did} \colorbox[RGB]{255,248,248}{\strut  a} \colorbox[RGB]{255,254,254}{\strut  t} \colorbox[RGB]{255,244,244}{\strut ribute} \colorbox[RGB]{255,238,238}{\strut  to} \colorbox[RGB]{255,184,184}{\strut  Shakespeare} \colorbox[RGB]{255,244,244}{\strut  on} \colorbox[RGB]{255,250,250}{\strut  the} \colorbox[RGB]{255,248,248}{\strut  screen} \colorbox[RGB]{255,222,222}{\strut .} \colorbox[RGB]{255,252,252}{\strut  They} \colorbox[RGB]{220,220,255}{\strut  should} \colorbox[RGB]{255,234,234}{\strut  have} \colorbox[RGB]{255,230,230}{\strut  been} \colorbox[RGB]{255,220,220}{\strut  giving} \colorbox[RGB]{255,244,244}{\strut  a} \colorbox[RGB]{255,250,250}{\strut  t} \colorbox[RGB]{255,222,222}{\strut ribute} \colorbox[RGB]{255,226,226}{\strut  to} \colorbox[RGB]{255,224,224}{\strut  Bran} \colorbox[RGB]{255,234,234}{\strut agh} \colorbox[RGB]{255,204,204}{\strut  for} \colorbox[RGB]{255,200,200}{\strut  bringing} \colorbox[RGB]{255,204,204}{\strut  us} \colorbox[RGB]{255,224,224}{\strut  one} \colorbox[RGB]{255,230,230}{\strut  of} \colorbox[RGB]{255,236,236}{\strut  the} \colorbox[RGB]{255,226,226}{\strut  greatest} \colorbox[RGB]{255,228,228}{\strut  films} \colorbox[RGB]{255,242,242}{\strut  of} \colorbox[RGB]{255,240,240}{\strut  all} \colorbox[RGB]{255,214,214}{\strut  time} \colorbox[RGB]{255,0,0}{\strut .}}}
     

%% file: figures/NLP/samples/7/heatmap_25_PE+LRP.tex
    \fbox{
    \parbox{\textwidth}{
    \setlength\fboxsep{0pt}
     \colorbox[RGB]{255,17,17}{\strut <s>} \colorbox[RGB]{255,248,248}{\strut  I} \colorbox[RGB]{254,254,255}{\strut  went} \colorbox[RGB]{255,252,252}{\strut  to} \colorbox[RGB]{255,252,252}{\strut  see} \colorbox[RGB]{255,224,224}{\strut  Ham} \colorbox[RGB]{255,222,222}{\strut let} \colorbox[RGB]{255,252,252}{\strut  because} \colorbox[RGB]{255,252,252}{\strut  I} \colorbox[RGB]{255,252,252}{\strut  was} \colorbox[RGB]{255,252,252}{\strut  in} \colorbox[RGB]{255,252,252}{\strut  between} \colorbox[RGB]{252,252,255}{\strut  jobs} \colorbox[RGB]{255,246,246}{\strut .} \colorbox[RGB]{255,250,250}{\strut  I} \colorbox[RGB]{254,254,255}{\strut  figured} \colorbox[RGB]{255,254,254}{\strut  } \colorbox[RGB]{254,254,255}{\strut 4} \colorbox[RGB]{255,254,254}{\strut  hours} \colorbox[RGB]{255,250,250}{\strut  would} \colorbox[RGB]{255,254,254}{\strut  be} \colorbox[RGB]{255,240,240}{\strut  great} \colorbox[RGB]{255,248,248}{\strut ,} \colorbox[RGB]{255,250,250}{\strut  I} \colorbox[RGB]{255,248,248}{\strut '} \colorbox[RGB]{255,252,252}{\strut ve} \colorbox[RGB]{255,250,250}{\strut  been} \colorbox[RGB]{255,252,252}{\strut  a} \colorbox[RGB]{255,242,242}{\strut  fan} \colorbox[RGB]{255,246,246}{\strut  of} \colorbox[RGB]{255,244,244}{\strut  Bran} \colorbox[RGB]{255,228,228}{\strut agh} \colorbox[RGB]{255,240,240}{\strut ;} \colorbox[RGB]{255,234,234}{\strut  Dead} \colorbox[RGB]{255,250,250}{\strut  Again} \colorbox[RGB]{255,250,250}{\strut ,} \colorbox[RGB]{255,236,236}{\strut  Henry} \colorbox[RGB]{255,252,252}{\strut  V} \colorbox[RGB]{255,246,246}{\strut .} \colorbox[RGB]{255,246,246}{\strut  I} \colorbox[RGB]{255,248,248}{\strut  was} \colorbox[RGB]{255,248,248}{\strut  completely} \colorbox[RGB]{255,254,254}{\strut  over} \colorbox[RGB]{255,248,248}{\strut wh} \colorbox[RGB]{255,254,254}{\strut el} \colorbox[RGB]{255,252,252}{\strut med} \colorbox[RGB]{255,244,244}{\strut  by} \colorbox[RGB]{255,248,248}{\strut  the} \colorbox[RGB]{255,254,254}{\strut  direction} \colorbox[RGB]{255,244,244}{\strut ,} \colorbox[RGB]{255,252,252}{\strut  acting} \colorbox[RGB]{255,248,248}{\strut ,} \colorbox[RGB]{255,248,248}{\strut  cinemat} \colorbox[RGB]{255,254,254}{\strut ography} \colorbox[RGB]{255,248,248}{\strut  that} \colorbox[RGB]{255,244,244}{\strut  this} \colorbox[RGB]{255,246,246}{\strut  film} \colorbox[RGB]{255,234,234}{\strut  captured} \colorbox[RGB]{255,222,222}{\strut .} \colorbox[RGB]{240,240,255}{\strut  Like} \colorbox[RGB]{255,246,246}{\strut  other} \colorbox[RGB]{252,252,255}{\strut  reviews} \colorbox[RGB]{255,240,240}{\strut  the} \colorbox[RGB]{255,254,254}{\strut  } \colorbox[RGB]{255,252,252}{\strut 4} \colorbox[RGB]{255,252,252}{\strut  hours} \colorbox[RGB]{255,250,250}{\strut  passes} \colorbox[RGB]{255,242,242}{\strut  swift} \colorbox[RGB]{255,246,246}{\strut ly} \colorbox[RGB]{255,248,248}{\strut .} \colorbox[RGB]{255,222,222}{\strut  Bran} \colorbox[RGB]{255,232,232}{\strut agh} \colorbox[RGB]{254,254,255}{\strut  doesn} \colorbox[RGB]{255,242,242}{\strut '} \colorbox[RGB]{255,254,254}{\strut t} \colorbox[RGB]{255,254,254}{\strut  play} \colorbox[RGB]{255,254,254}{\strut  Ham} \colorbox[RGB]{255,250,250}{\strut let} \colorbox[RGB]{255,250,250}{\strut ,} \colorbox[RGB]{255,250,250}{\strut  he} \colorbox[RGB]{255,252,252}{\strut  is} \colorbox[RGB]{255,252,252}{\strut  Ham} \colorbox[RGB]{255,250,250}{\strut let} \colorbox[RGB]{255,250,250}{\strut ,} \colorbox[RGB]{255,252,252}{\strut  he} \colorbox[RGB]{255,254,254}{\strut  was} \colorbox[RGB]{255,244,244}{\strut  born} \colorbox[RGB]{255,244,244}{\strut  for} \colorbox[RGB]{255,246,246}{\strut  this} \colorbox[RGB]{255,228,228}{\strut .} \colorbox[RGB]{255,246,246}{\strut  When} \colorbox[RGB]{255,248,248}{\strut  I} \colorbox[RGB]{255,252,252}{\strut  watch} \colorbox[RGB]{255,244,244}{\strut  this} \colorbox[RGB]{255,246,246}{\strut  film} \colorbox[RGB]{255,248,248}{\strut  I} \colorbox[RGB]{254,254,255}{\strut '} \colorbox[RGB]{254,254,255}{\strut m} \colorbox[RGB]{255,252,252}{\strut  constantly} \colorbox[RGB]{255,254,254}{\strut  trying} \colorbox[RGB]{255,252,252}{\strut  to} \colorbox[RGB]{255,252,252}{\strut  find} \colorbox[RGB]{255,248,248}{\strut  fault} \colorbox[RGB]{255,252,252}{\strut s} \colorbox[RGB]{255,250,250}{\strut ,} \colorbox[RGB]{255,254,254}{\strut  I} \colorbox[RGB]{255,252,252}{\strut '} \colorbox[RGB]{255,254,254}{\strut ve} \colorbox[RGB]{255,254,254}{\strut  looked} \colorbox[RGB]{255,254,254}{\strut  at} \colorbox[RGB]{255,252,252}{\strut  the} \colorbox[RGB]{255,254,254}{\strut  go} \colorbox[RGB]{255,252,252}{\strut of} \colorbox[RGB]{255,252,252}{\strut s} \colorbox[RGB]{255,250,250}{\strut  and} \colorbox[RGB]{255,254,254}{\strut  haven} \colorbox[RGB]{255,250,250}{\strut '} \colorbox[RGB]{254,254,255}{\strut t} \colorbox[RGB]{255,252,252}{\strut  noticed} \colorbox[RGB]{255,254,254}{\strut  them} \colorbox[RGB]{255,240,240}{\strut .} \colorbox[RGB]{248,248,255}{\strut  How} \colorbox[RGB]{255,246,246}{\strut  he} \colorbox[RGB]{255,254,254}{\strut  was} \colorbox[RGB]{255,248,248}{\strut  able} \colorbox[RGB]{255,254,254}{\strut  to} \colorbox[RGB]{254,254,255}{\strut  move} \colorbox[RGB]{255,254,254}{\strut  the} \colorbox[RGB]{255,254,254}{\strut  camera} \colorbox[RGB]{255,254,254}{\strut  in} \colorbox[RGB]{255,252,252}{\strut  and} \colorbox[RGB]{255,254,254}{\strut  out} \colorbox[RGB]{255,254,254}{\strut  of} \colorbox[RGB]{255,252,252}{\strut  the} \colorbox[RGB]{255,250,250}{\strut  Hall} \colorbox[RGB]{255,250,250}{\strut  with} \colorbox[RGB]{255,254,254}{\strut  all} \colorbox[RGB]{255,252,252}{\strut  the} \colorbox[RGB]{255,252,252}{\strut  mirror} \colorbox[RGB]{255,254,254}{\strut s} \colorbox[RGB]{255,252,252}{\strut  is} \colorbox[RGB]{255,252,252}{\strut  a} \colorbox[RGB]{252,252,255}{\strut  mystery} \colorbox[RGB]{255,254,254}{\strut  to} \colorbox[RGB]{255,250,250}{\strut  me} \colorbox[RGB]{255,248,248}{\strut .} \colorbox[RGB]{252,252,255}{\strut  This} \colorbox[RGB]{254,254,255}{\strut  movie} \colorbox[RGB]{255,254,254}{\strut  was} \colorbox[RGB]{255,248,248}{\strut  shot} \colorbox[RGB]{255,252,252}{\strut  in} \colorbox[RGB]{255,254,254}{\strut  } \colorbox[RGB]{255,252,252}{\strut 7} \colorbox[RGB]{255,252,252}{\strut 0} \colorbox[RGB]{255,246,246}{\strut  mil} \colorbox[RGB]{255,240,240}{\strut .} \colorbox[RGB]{255,250,250}{\strut  It} \colorbox[RGB]{255,232,232}{\strut '} \colorbox[RGB]{255,254,254}{\strut s} \colorbox[RGB]{255,254,254}{\strut  a} \colorbox[RGB]{255,254,254}{\strut  shame} \colorbox[RGB]{255,250,250}{\strut  that} \colorbox[RGB]{255,242,242}{\strut  Columbia} \colorbox[RGB]{254,254,255}{\strut  hasn} \colorbox[RGB]{255,254,254}{\strut '} \colorbox[RGB]{255,254,254}{\strut t} \colorbox[RGB]{255,252,252}{\strut  released} \colorbox[RGB]{255,254,254}{\strut  a} \colorbox[RGB]{255,252,252}{\strut  W} \colorbox[RGB]{255,250,250}{\strut ides} \colorbox[RGB]{255,240,240}{\strut creen} \colorbox[RGB]{255,250,250}{\strut  version} \colorbox[RGB]{255,252,252}{\strut  of} \colorbox[RGB]{255,248,248}{\strut  this} \colorbox[RGB]{255,248,248}{\strut  on} \colorbox[RGB]{255,242,242}{\strut  V} \colorbox[RGB]{255,232,232}{\strut HS} \colorbox[RGB]{255,242,242}{\strut .} \colorbox[RGB]{255,248,248}{\strut  I} \colorbox[RGB]{255,250,250}{\strut  own} \colorbox[RGB]{255,252,252}{\strut  a} \colorbox[RGB]{255,248,248}{\strut  DVD} \colorbox[RGB]{255,250,250}{\strut  player} \colorbox[RGB]{255,248,248}{\strut ,} \colorbox[RGB]{255,246,246}{\strut  and} \colorbox[RGB]{255,246,246}{\strut  I} \colorbox[RGB]{255,218,218}{\strut '} \colorbox[RGB]{255,254,254}{\strut d} \colorbox[RGB]{255,254,254}{\strut  take} \colorbox[RGB]{255,248,248}{\strut  this} \colorbox[RGB]{255,250,250}{\strut  over} \colorbox[RGB]{255,254,254}{\strut  T} \colorbox[RGB]{255,254,254}{\strut itan} \colorbox[RGB]{255,254,254}{\strut ic} \colorbox[RGB]{255,252,252}{\strut  any} \colorbox[RGB]{255,252,252}{\strut  day} \colorbox[RGB]{255,242,242}{\strut .} \colorbox[RGB]{255,246,246}{\strut  So} \colorbox[RGB]{255,238,238}{\strut  Columbia} \colorbox[RGB]{255,242,242}{\strut  if} \colorbox[RGB]{255,246,246}{\strut  you} \colorbox[RGB]{255,240,240}{\strut '} \colorbox[RGB]{255,254,254}{\strut re} \colorbox[RGB]{255,244,244}{\strut  listening} \colorbox[RGB]{254,254,255}{\strut  put} \colorbox[RGB]{255,244,244}{\strut  this} \colorbox[RGB]{255,246,246}{\strut  film} \colorbox[RGB]{255,252,252}{\strut  out} \colorbox[RGB]{255,252,252}{\strut  the} \colorbox[RGB]{255,246,246}{\strut  way} \colorbox[RGB]{255,252,252}{\strut  it} \colorbox[RGB]{254,254,255}{\strut  should} \colorbox[RGB]{255,250,250}{\strut  be} \colorbox[RGB]{255,248,248}{\strut  watched} \colorbox[RGB]{255,152,152}{\strut !} \colorbox[RGB]{250,250,255}{\strut  And} \colorbox[RGB]{255,246,246}{\strut  I} \colorbox[RGB]{255,250,250}{\strut  don} \colorbox[RGB]{255,234,234}{\strut '} \colorbox[RGB]{255,250,250}{\strut t} \colorbox[RGB]{255,246,246}{\strut  know} \colorbox[RGB]{255,252,252}{\strut  what} \colorbox[RGB]{255,250,250}{\strut  happened} \colorbox[RGB]{255,252,252}{\strut  at} \colorbox[RGB]{255,250,250}{\strut  the} \colorbox[RGB]{255,252,252}{\strut  O} \colorbox[RGB]{255,254,254}{\strut sc} \colorbox[RGB]{255,248,248}{\strut ars} \colorbox[RGB]{255,240,240}{\strut .} \colorbox[RGB]{255,248,248}{\strut  This} \colorbox[RGB]{246,246,255}{\strut  should} \colorbox[RGB]{255,252,252}{\strut  have} \colorbox[RGB]{255,248,248}{\strut  swe} \colorbox[RGB]{255,246,246}{\strut pt} \colorbox[RGB]{255,250,250}{\strut  Best} \colorbox[RGB]{255,250,250}{\strut  Picture} \colorbox[RGB]{255,242,242}{\strut ,} \colorbox[RGB]{255,254,254}{\strut  Best} \colorbox[RGB]{255,254,254}{\strut  A} \colorbox[RGB]{255,252,252}{\strut ctor} \colorbox[RGB]{255,250,250}{\strut ,} \colorbox[RGB]{255,254,254}{\strut  Best} \colorbox[RGB]{255,254,254}{\strut  D} \colorbox[RGB]{255,254,254}{\strut irection} \colorbox[RGB]{255,250,250}{\strut ,} \colorbox[RGB]{255,250,250}{\strut  best} \colorbox[RGB]{255,252,252}{\strut  cinemat} \colorbox[RGB]{255,252,252}{\strut ography} \colorbox[RGB]{255,238,238}{\strut .} \colorbox[RGB]{248,248,255}{\strut  What} \colorbox[RGB]{255,242,242}{\strut  films} \colorbox[RGB]{255,244,244}{\strut  were} \colorbox[RGB]{255,246,246}{\strut  they} \colorbox[RGB]{255,244,244}{\strut  watching} \colorbox[RGB]{255,230,230}{\strut ?} \colorbox[RGB]{255,232,232}{\strut  I} \colorbox[RGB]{255,234,234}{\strut  felt} \colorbox[RGB]{255,254,254}{\strut  sorry} \colorbox[RGB]{255,242,242}{\strut  for} \colorbox[RGB]{255,230,230}{\strut  Bran} \colorbox[RGB]{255,240,240}{\strut agh} \colorbox[RGB]{255,240,240}{\strut  at} \colorbox[RGB]{255,250,250}{\strut  the} \colorbox[RGB]{255,246,246}{\strut  O} \colorbox[RGB]{255,248,248}{\strut sc} \colorbox[RGB]{255,240,240}{\strut ars} \colorbox[RGB]{255,220,220}{\strut  when} \colorbox[RGB]{255,244,244}{\strut  he} \colorbox[RGB]{255,248,248}{\strut  did} \colorbox[RGB]{255,248,248}{\strut  a} \colorbox[RGB]{255,254,254}{\strut  t} \colorbox[RGB]{255,244,244}{\strut ribute} \colorbox[RGB]{255,238,238}{\strut  to} \colorbox[RGB]{255,184,184}{\strut  Shakespeare} \colorbox[RGB]{255,244,244}{\strut  on} \colorbox[RGB]{255,250,250}{\strut  the} \colorbox[RGB]{255,248,248}{\strut  screen} \colorbox[RGB]{255,222,222}{\strut .} \colorbox[RGB]{255,252,252}{\strut  They} \colorbox[RGB]{220,220,255}{\strut  should} \colorbox[RGB]{255,234,234}{\strut  have} \colorbox[RGB]{255,230,230}{\strut  been} \colorbox[RGB]{255,220,220}{\strut  giving} \colorbox[RGB]{255,244,244}{\strut  a} \colorbox[RGB]{255,250,250}{\strut  t} \colorbox[RGB]{255,222,222}{\strut ribute} \colorbox[RGB]{255,226,226}{\strut  to} \colorbox[RGB]{255,224,224}{\strut  Bran} \colorbox[RGB]{255,234,234}{\strut agh} \colorbox[RGB]{255,204,204}{\strut  for} \colorbox[RGB]{255,200,200}{\strut  bringing} \colorbox[RGB]{255,204,204}{\strut  us} \colorbox[RGB]{255,224,224}{\strut  one} \colorbox[RGB]{255,230,230}{\strut  of} \colorbox[RGB]{255,236,236}{\strut  the} \colorbox[RGB]{255,226,226}{\strut  greatest} \colorbox[RGB]{255,228,228}{\strut  films} \colorbox[RGB]{255,242,242}{\strut  of} \colorbox[RGB]{255,240,240}{\strut  all} \colorbox[RGB]{255,214,214}{\strut  time} \colorbox[RGB]{255,0,0}{\strut .}}}